\documentclass{article}

\PassOptionsToPackage{numbers, compress}{natbib}

\usepackage[final]{neurips_2019}

\usepackage{times}
\usepackage{microtype} 
\usepackage{graphicx} 
\usepackage{subfigure}
\usepackage{wrapfig}
\usepackage{balance}

\usepackage{algorithm}
\usepackage{algorithmic}

\renewcommand{\algorithmicrequire}{\textbf{Initialize:}}
\renewcommand{\algorithmicensure}{\textbf{Iteration:}}
\renewcommand{\algorithmicinput}{\textbf{Input:}}
\renewcommand{\algorithmicoutput}{\textbf{Output:}}

\usepackage{hyperref}

\usepackage{helvet}
\usepackage{courier}

\usepackage{makecell}
\usepackage{graphicx}
\usepackage{subfigure}
\usepackage{color}
\usepackage{amsfonts}
\usepackage{amsmath}
\usepackage{amssymb}

\usepackage{multirow}
\usepackage{booktabs}

\def\ie{\mbox{\textit{i.e.}, }}
\def\eg{\mbox{\textit{e.g.}, }}
\def\wrt{\mbox{\textit{w.r.t. }}}

\def\mC{{\mathcal C}}
\def\mD{{\mathcal D}}

\def\mF{{\mathcal F}}
\def\mG{{\mathcal G}}

\def\mI{{\mathcal I}}
\def\mJ{{\mathcal J}}
\def\mK{{\mathcal K}}

\def\mM{{\mathcal M}}

\def\mP{{\mathcal P}}

\def\mR{{\mathcal R}}

\def\mW{{\mathcal W}}
\def\mX{{\mathcal X}}
\def\mY{{\mathcal Y}}

\DeclareMathAlphabet\mathbfcal{OMS}{cmsy}{b}{n}

\def\0{{\bf 0}}
\def\1{{\bf 1}}

\def\bx{{\bf x}}
\def\by{{\bf y}}
\def\bz{{\bf z}}

\def\mmE{{\mathbb E}}
\def\mmP{{\mathbb P}}
\def\mmQ{{\mathbb Q}}

\def\mmR{{\mathbb R}}

\def\st{{\mathrm{s.t.}}}

\def\bx{{\bf x}}

\def\by{{\bf y}}

\def\bz{{\bf z}}

\def\st{{\mathrm{s.t.}}}

\usepackage{ntheorem}
\newtheorem{coll}{Corollary}
\newtheorem{deftn}{Definition}
\newtheorem{thm}{Theorem}
\newtheorem*{*thm}{Theorem}
\newtheorem{prop}{Proposition}
\newtheorem{lemma}{Lemma}
\newtheorem*{*lemma}{Lemma}
\newtheorem{remark}{Remark}
\newtheorem{ass}{Assumption}
\newtheorem{prob}{Problem}
\newenvironment*{proof}{\textbf{Proof}\quad}{\hfill $\square$\par}

\usepackage{array}
\usepackage{tabu}
\makeatletter
\newcommand{\nipstophline}{%
	\noalign {\ifnum 0=`}\fi \hrule height 4pt
	\futurelet \reserved@a \@xhline
}
\newcommand{\nipsbottomhline}{%
	\noalign {\ifnum 0=`}\fi \hrule height 1pt
	\futurelet \reserved@a \@xhline
}
\usepackage{enumitem}

\newcommand{\yes}{\textcolor[rgb]{0,0.6,0}{\large{\checkmark}}}
\newcommand{\no}{\textcolor[rgb]{0.6,0,0}{\large{{\bf\times}}}}

\usepackage[utf8]{inputenc} 
\usepackage[T1]{fontenc}    
\usepackage{nicefrac}       

\title{Multi-marginal Wasserstein GAN}

\author{
	Jiezhang Cao$^1$\thanks{Authors contributed equally.}~~, Langyuan Mo$^{1*}$, Yifan Zhang$^{1}$, Kui Jia$^{1}$,  Chunhua Shen$^3$, Mingkui Tan$^{1,2*}$\thanks{Corresponding author.}  \\
	$^1$South China University of Technology, $^2$Peng Cheng Laboratory, $^3$The University of Adelaide \\
	\{secaojiezhang, selymo, sezyifan\}@mail.scut.edu.cn \\ \{mingkuitan, kuijia\}@scut.edu.cn, chunhua.shen@adelaide.edu.au
}

\begin{document}
\maketitle

\def\kuired{\textcolor{black}}
\def\yifan{\textcolor{black}}
\def\warming{\textcolor{black}}
\def\cao{\textcolor{black}}
\def\kui{\textcolor{black}}
\def\mo{\textcolor{black}}
\def\langyuan{\textcolor{black}}
\def\jie{\textcolor{black}}

\begin{abstract}
	Multiple marginal matching problem aims at learning mappings to match a source domain to multiple target domains and it has attracted great attention in many applications, such as multi-domain image translation. However, addressing this problem has two critical challenges: (i) Measuring the multi-marginal distance among different domains is very intractable; (ii) It is very difficult to exploit cross-domain correlations to match the target domain distributions. In this paper, we propose a novel Multi-marginal Wasserstein GAN (MWGAN) to minimize Wasserstein distance among domains. Specifically, with the help of multi-marginal optimal transport theory, we develop a new adversarial objective function with inner- and inter-domain constraints to exploit cross-domain correlations. Moreover, we theoretically analyze the generalization performance of MWGAN, and empirically evaluate it on the balanced and imbalanced translation tasks. Extensive experiments on toy and real-world datasets demonstrate the effectiveness of MWGAN.

\end{abstract}

\section{Introduction}
Multiple marginal matching (\cao{M$^3$}) problem aims to map an input image (source domain) to multiple target domains (see Figure \ref{subfig:example}), and it has been applied in computer vision, 
\eg 
multi-domain image translation \cite{choi2018stargan, he2017attgan, hui2018unsupervised}. 
In practice, the unsupervised image translation \cite{liu2018ufdn} gains particular interest because of its label-free property.  
However, due to the lack of corresponding images, \kui{this task is extremely hard to learn stable mappings to match a source distribution to multiple target distributions.}
Recently, some methods~\cite{choi2018stargan, liu2018ufdn} address \cao{M$^3$} problem, which, however, face two main challenges. 

First, existing methods \cao{often neglect to} jointly optimize the multi-marginal distance among domains, which cannot guarantee the generalization performance of methods and may lead to distribution mismatching issue.
Recently, CycleGAN \cite{zhu2017unpaired} and UNIT \cite{liu2017unsupervised} repeatedly optimize every pair of two different domains separately (see Figure \ref{subfig:method_comparison}). 
In this sense, they are computationally expensive and may have poor generalization performance. 
Moreover, UFDN \cite{liu2018ufdn} and StarGAN \cite{choi2018stargan} essentially measure the distance between an input distribution and a mixture of all target distributions (see Figure \ref{subfig:method_comparison}).
As a result, they may suffer from distribution mismatching issue.
Therefore, it is necessary to explore a new method to measure and optimize the multi-marginal distance.

\cao{
	Second, it is very challenging to exploit the cross-domain correlations to match target domains. 
	Existing methods \cite{zhu2017unpaired, liu2018ufdn} only focus on the correlations between the source and target domains, since they measure the distance between two distributions (see Figure \ref{subfig:method_comparison}).
	However, these methods often ignore the correlations among target domains, and thus they are hard to fully capture information to improve the performance. 
	Moreover, when the source and target domains are significantly different, or the number of target domains is large, the translation task turns to be difficult for existing methods to exploit the cross-domain correlations.
}


\langyuan{In this paper, we seek to use multi-marginal Wasserstein distance to solve M$^3$ problem, but directly optimizing it is intractable. Therefore, we develop a new dual formulation to make it tractable and propose a novel multi-marginal Wasserstein GAN (MWGAN) by enforcing inner- and inter-domain constraints to exploit the correlations among domains.}


The contributions of this paper are summarized as follows:
\vspace{-5pt}
\begin{itemize}[leftmargin=*]\setlength{\itemsep}{0pt}
	\item 
	\cao{We propose a novel GAN method (called MWGAN) to optimize a feasible multi-marginal distance among different domains.}
	\cao{MWGAN overcomes the limitations of existing methods by alleviating the distribution mismatching issue and exploiting cross-domain correlations.}%
	\item
	\cao{We define and analyze the generalization of our proposed method for the multiple domain translation task, which is more important than existing generalization analyses \cite{galanti2018generalization, pan2018theoretical} studying only on two domains and non-trivial for multiple domains. }
	\item 
	We empirically show that MWGAN is able to solve the imbalanced image translation task well \cao{when the source and target domains are significantly different.}
	Extensive experiments on toy and real-world datasets demonstrate the effectiveness of our proposed method. 
\end{itemize}

\begin{figure}[tp]
	\centering
	\subfigure[The Edge$\rightarrow$CelebA image translation task.]{
		\label{subfig:example}
		\includegraphics[width=0.47\linewidth]{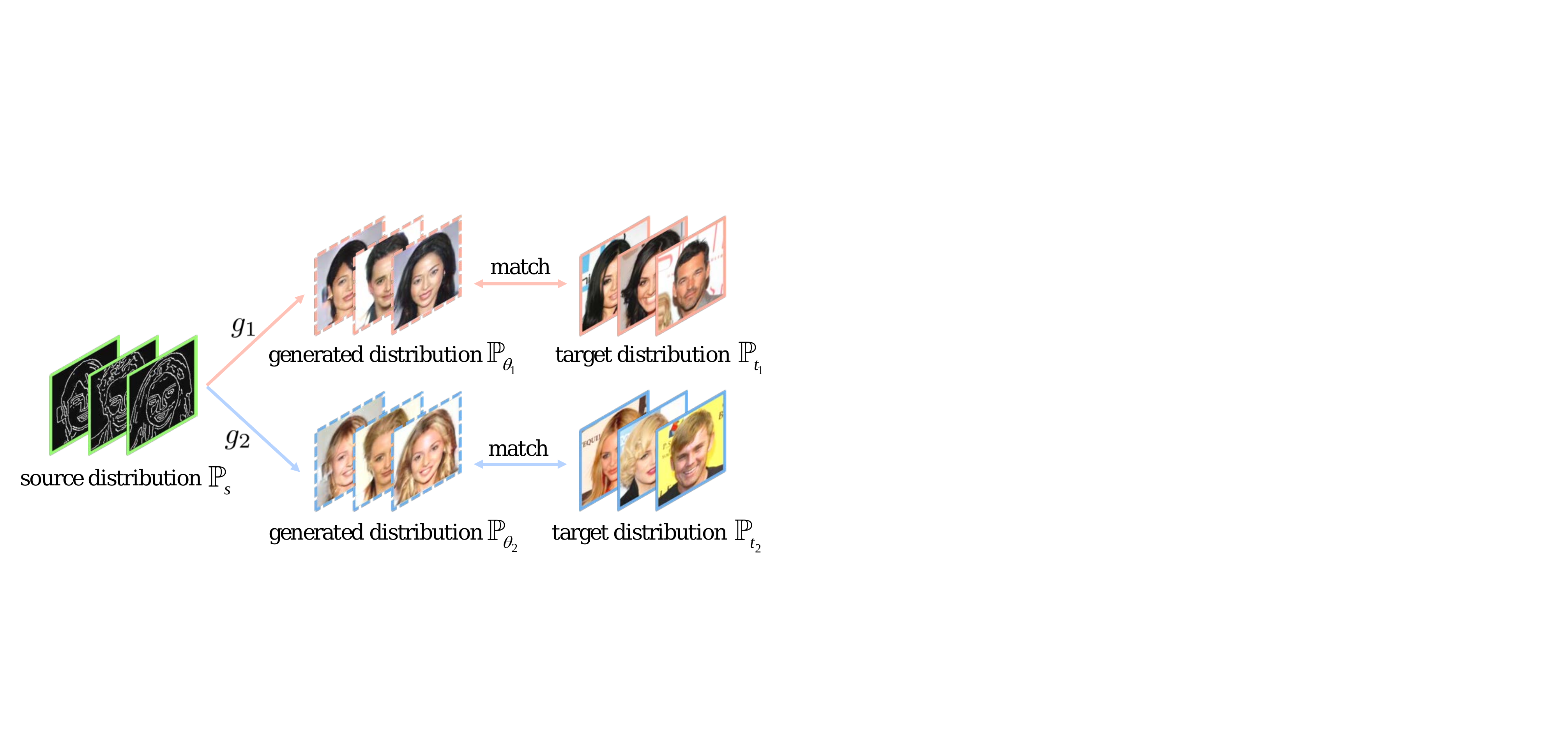}
	}
	\hspace{0.2in}
	\subfigure[Comparisons of different distribution measures.]{
		\label{subfig:method_comparison}
		\includegraphics[width=0.45\linewidth]{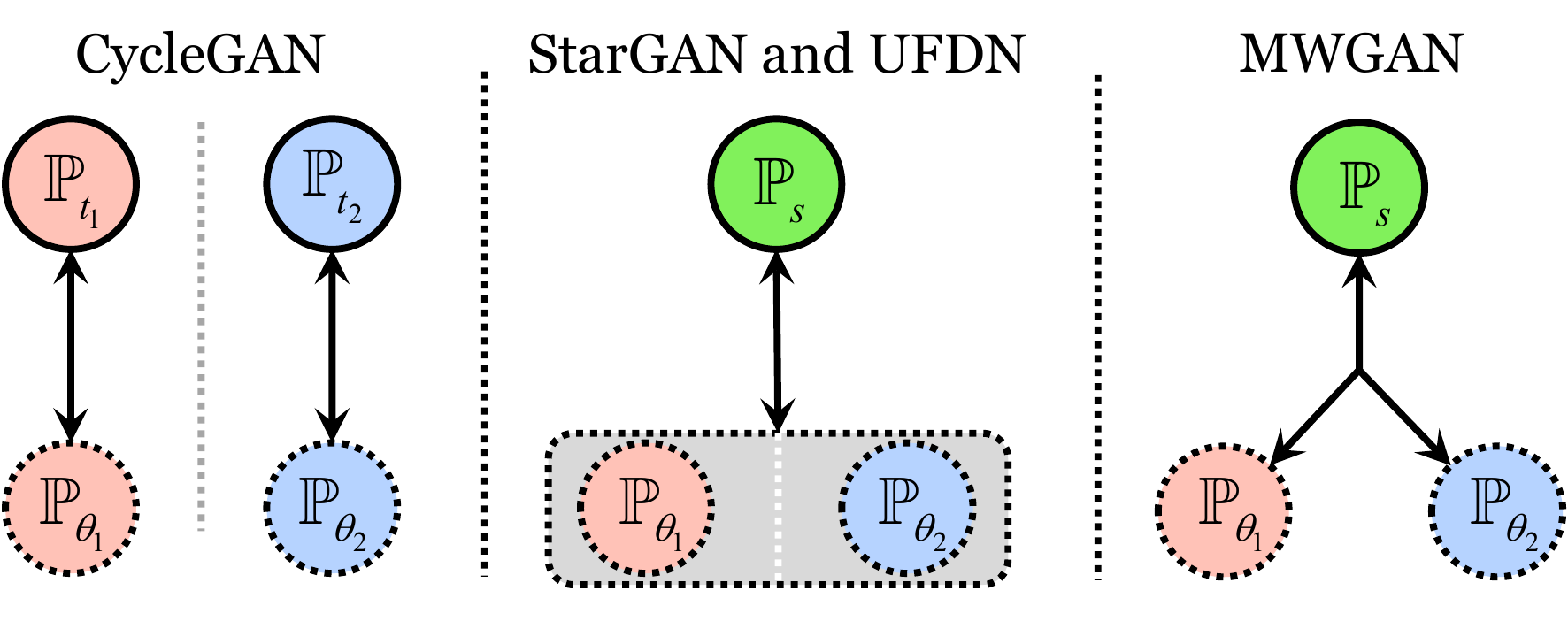}
	}
	\caption{An example of M$^3$ problem and comparisons of existing methods.
		(a) For the Edge$\rightarrow$CelebA task, we aim to learn mappings to match a source distribution (\ie Edge images) to the target distributions (\ie black and blond hair images).
		(b)
		\cao{Left: we employ CycleGAN multiple times to measure the distance between every generated distribution and its corresponding target distribution.}
		Middle: StarGAN and UFDN measure the distance between  $ \mmP_s $ and a mixed distribution of $ \mmP_{\theta_1} $ and $ \mmP_{\theta_2} $; Right: MWGAN jointly measures Wasserstein distance among $ \mmP_s $, $ \mmP_{\theta_1} $ and $ \mmP_{\theta_2} $. (Dotted circle: the generated distributions, solid circle: the real source or target distributions, double-headed arrow: distribution divergence, different colors represent different domains.)
	}
	\label{fig:comparison}
\end{figure}

\section{Related Work} \label{sec:related_work}
\textbf{Generative adversarial networks (GANs).}
Deep neural networks have theoretical and experimental explorations \cite{cao2017flatness, guo2019nat, zeng2019breaking, zeng2019graph, zhuang2018discrimination}.
In particular, GANs \cite{goodfellow2014gans} have been successfully applied in computer vision tasks, such as image generation \cite{arjovsky2017wasserstein, cao2018lccgan, Gulrajani2017gangp, guo2019auto}, image translation \cite{almahairi2018acyclegan, choi2018stargan, guo2018dual} and video prediction \cite{mathieu2015deep}. 
Specifically, 
a generator tries to produce realistic samples, while a discriminator tries to distinguish between generated data and real data.
\langyuan{Recently, some studies try to improve the quality \cite{brock2018large, chen2018attention, karras2018pggan} and diversity \cite{wang2019evolutionary} of generated images,
	and improve the mechanism of GANs \cite{Adler2018BanachGAN, farnia2018a, Kevin2017regularization, sanjabi2018on} to deal with the unstable training and mode collapse problems.} 

\textbf{Multi-domain image translation.}
M$^3$ problem can be applied in domain adaptation \cite{xie2018learning} and image translation \cite{Kazemi2018unsupervised, zhu2017toward}.
CycleGAN \cite{zhu2017unpaired}, DiscoGAN \cite{kim2017learning}, DualGAN \cite{yi2017dualgan} and UNIT \cite{liu2017unsupervised} are proposed to address two-domain image translation task. 
However, in Figure \ref{subfig:method_comparison}, these methods measure the distance between every pair of distributions multiple times, which is computationally expensive when applied to the multi-domain image translation task.  
Recently, StarGAN \cite{choi2018stargan} and AttGAN \cite{he2017attgan} use a single model to perform multi-domain image translation. 
UFDN \cite{liu2018ufdn} translates images by learning domain-invariant representation {for} cross-domains.
Essentially, the above three methods are two-domain image translation methods because they measure the distance between an input distribution and a uniform mixture of other target distributions (see Figure \ref{subfig:method_comparison}).
Therefore, these methods may suffer from distribution mismatching issue and obtain misleading feedback for updating models when the source and target domains are significantly different.
In addition, we discuss the difference between some GAN methods in Section \ref{sec:diff_gan} in supplementary materials. 

\section{Problem Definition}


\paragraph{Notation.}
We use calligraphic letters (\textit{e.g.}, $ \mX $) to denote space, capital letters (\textit{e.g.}, $ X $) to denote random variables, and bold lower case letter (\textit{e.g.}, $ \bx $) to denote the corresponding values. 
Let $ \mD {=} (\mX , \mmP) $ be the domain, $ \mmP $ or $ \mu $ be the marginal distribution over $ \mX $ and $ \mP(\mX) $ be the set of all the probability measures over $ \mX $.
For convenience, let $ \mX {=} \mmR^d $, and let $ \mI {=} \{ 0, ..., N \} $ and $ [N]{=}\{1, ..., N\} $.%

\textbf{Multiple marginal matching (M$^3$) problem.}
In this paper, M$^3$ problem aims to learn mappings to match a source domain to multiple target domains. 
For simplicity, we consider one source domain $ \mD_s{=}\{\mX, \mmP_s\} $ and $ N $ target domains $ \mD_i{=}\{\mX, \mmP_{t_i}\}, i {\in} [N] $, where $ \mmP_s $ is the source distribution, and $ \mmP_{t_i} $ is the $ i $-th real target distribution.
Let $ g_i, i {\in} [N] $ be the generative models parameterized by $ \theta_i $,  
and $ \mmP_{\theta_i} $ be the generated distribution in the $ i $-th target domain. 
In this problem, the goal is to learn multiple generative models such that each generated distribution $ \mmP_{\theta_i} $ in the $ i $-th target domain can be close to the corresponding real target distribution $ \mmP_{t_i} $ (see Figure \ref{subfig:example}). 


\textbf{Optimal transport (OT) theory.}
Recently, OT \cite{villani2008optimal} theory has attracted great attention in many applications \cite{arjovsky2017wasserstein, yan2019oversampling}.
\langyuan{
	Directly solving the primal formulation of OT \cite{santambrogio2015optimal} 
	might be intractable \cite{genevay2018learning}. 
}
To address this, we consider the dual formulation of the multi-marginal OT problem as follows. 

\begin{prob} \textbf{\emph{(Dual problem \cite{santambrogio2015optimal})}  }\label{problem:MK-D}
	Given $ N{+}1 $ marginals $ \mu_i {\in} \mP(\mX) $, potential functions $ f_i, i {\in} \mI $, and a cost function $ c(X^{(0)}, \ldots, X^{(N)}): \mmR^{d(N{+}1)} {\to} \mmR $, the dual Kantorovich problem can be defined as:
	\begin{small}
		\begin{align}
			W(\mu_0, {...}, \mu_{N}) = 
			\sup_{f_i} \sum\nolimits_{i} \int f_i \left(X^{(i)}\right) d \mu_i \left(X^{(i)}\right), \;\st \sum\nolimits_{i} f_i \left(X^{(i)}\right) {\leq} c \left(X^{(0)}, {...}, X^{(N)} \right). 
		\end{align}
	\end{small}
\end{prob}
\cao{In practice, we optimize the discrete case of Problem \ref{problem:MK-D}.
	Specifically, given samples \begin{small}{$ \{\bx_j^{(0)}\}_{j{\in}\mJ_0} $} \end{small} and \begin{small}$ \{\bx_j^{(i)}\}_{j{\in}\mJ_i} $\end{small} drawn from source domain distribution $ \mmP_s $ and generated target distributions \begin{small}$ \mmP_{\theta_i}, i {\in} [N] $\end{small}, respectively, where $ \mJ_i $ is an index set and $ n_i {=} |\mJ_i| $ is the number of samples, we have: }
\begin{prob}\textbf{\emph{(Discrete dual problem)}} \label{problem:dist_D}
	Let $ F {=} \{f_0, \ldots, f_N\} $ be the set of Kantorovich potentials, then the discrete dual problem $\hat{h} (F)$ can be defined as:
	\begin{small}
		\begin{align} \label{eqn:dist_D}
			\max_F \; \hat{h} (F) = \sum\nolimits_{i} \frac{1}{n_i} \sum\nolimits_{j \in \mJ_i} f_i\left(\bx_{j}^{(i)} \right), \quad\st \sum\nolimits_{i} f_i \left(\bx^{(i)}_{k_i} \right) \leq  c \left(\bx^{(0)}_{k_0}, \ldots, \bx^{(N)}_{k_N} \right), \forall\; k_i \in [n_i]. 
		\end{align}
	\end{small}
\end{prob}
\cao{Unfortunately, it is challenging to optimize Problem \ref{problem:dist_D} due to the intractable inequality constraints and multiple potential functions.  
	To address this, we seek to propose a new optimization method.}

\section{Multi-marginal Wasserstein GAN}

\subsection{A New Dual Formulation}\label{dualformulate}   
\jie{For two domains, WGAN \cite{arjovsky2017wasserstein} solves Problem \ref{problem:dist_D} by setting potential functions as $ f_0 {=} f $ and $ f_1 {=} {-}f $.} 
\jie{However, it is hard to extend WGAN to multiple domains. }
\jie{To address this, we propose a new dual formulation in order to optimize Problem \ref{problem:dist_D}}. 
To this end, we use a shared potential in Problem \ref{problem:dist_D}, which is supported by empirical and theoretical evidence.
In the multi-domain image translation task, the domains are often correlated, and thus share similar properties and differ only in details (see Figure \ref{subfig:example}).
The cross-domain correlations can be exploited by the shared potential function (see Section \ref{sec:one_potential} in supplementary materials).
More importantly, the optimal objectives of Problem \ref{problem:dist_D} and the following problem can be equal under some conditions (see Section \ref{sec:equiv_thm} in supplementary materials).

\begin{prob}\label{problem:newD}
	Let $ F_{\lambda} {=} \{\lambda_0 f, \ldots, \lambda_N f\} $ be Kantorovich potentials, 
	then we define dual problem as: 
	\begin{small}
		\begin{align} \label{obj:newD}
			&\max_{F_{\lambda}} \quad\hat{h} (F_{\lambda}) = \sum\nolimits_{i} \frac{\lambda_i}{n_i} \sum\nolimits_{j \in \mJ_i} f\left(\bx_{j}^{(i)} \right), \quad\,\,\, \st \sum\nolimits_{i} \lambda_i f \left(\bx^{(i)}_{k_i} \right) \leq  c \left(\bx^{(0)}_{k_0}, \ldots, \bx^{(N)}_{k_N} \right), \forall k_i {\in} [n_i].
		\end{align}
	\end{small}
\end{prob}
To further build the relationship \kuired{between Problem \ref{problem:dist_D} and Problem \ref{problem:newD},} we have the following theorem so that Problem \ref{problem:newD} can be optimized well by GAN-based methods (see Subsection \ref{subsec:mwgan}).
\begin{thm} \label{thm:solution_exist}
	Suppose the domains are connected, the cost function $ c $ is continuously differentiable and each $ \mu_i $ is absolutely continuous. 
	If $ (f_0, \ldots, f_N) $ and $ (\lambda_0f, \ldots, \lambda_{N}f ) $ are solutions to Problem \ref{problem:MK-D}, then there exist some constants $ \varepsilon_i $ for each $ i \in \mI $ such that $ \sum_{i} \varepsilon_i = 0 $ and $ f_i = \lambda_i f + \varepsilon_i $. 
\end{thm}

\begin{remark}
	\cao{
		From Theorem \ref{thm:solution_exist}, if we train a shared function $f$ to obtain a solution  \yifan{of} Problem \ref{problem:MK-D}, 
		we have an equivalent Wasserstein distance, \ie $\sum_i f_i {=} \sum_i \lambda_i f$ regardless of whatever the value $\varepsilon_i$ is.
		Therefore, we are able to optimize Problem \ref{problem:newD} instead of intractable Problem \ref{problem:dist_D} in practice.}
\end{remark}

\begin{algorithm}[t]
	\caption{Multi-marginal WGAN.}
	\label{alg:mwgan}
	\begin{algorithmic}[1]
		\begin{small}
			\INPUT  Training data $\{ {\bx}_j^{} \}_{j{=}1}^{n_0} $ in the initial domain, $\{ \hat{\bx}_j^{(i)} \}_{j{=}1}^{n_i} $ in the $ i $-th target domain; batch size $ m_{bs} $; 
			the number of iterations of the discriminator per generator iteration $ n_{\text{critic}} $; Uniform distribution $ U[0, 1] $.
			\OUTPUT The discriminator $ f $, the generators $ \{g_i\}_{i {\in} [N]} $ and the classifier $ \phi $
			\WHILE{not converged} 
			\FOR{ $ t= 0, \ldots, n_{\text{critic}} $} 
			\STATE Sample $ \bx {\sim} \hat{\mmP}_s $ and $\hat{\bx} {\sim} \hat{\mmP}_{\theta_i}, \forall i $, and $ \tilde{\bx} \leftarrow \rho \bx + (1-\rho) \hat{\bx} $, where $ \rho {\sim} U[0, 1] $ \\ 
			\STATE Update $ f $ by ascending the gradient: $ \nabla_w [ \mathop{\mmE}\nolimits_{\bx{\sim} \hat{\mmP}_s} \left[ f \left(\bx\right) \right] {-} \sum\nolimits_{i} \lambda_i^+  \mathop{\mmE}\nolimits_{\hat{\bx}{\sim} \hat{\mmP}_{\theta_i}} \left[ f \left(\hat{\bx}\right) \right] {+} \mR_{\tau}(f) ] $ 
			\STATE Update classifier $ \phi $ by descending the gradient {{$ \nabla_v [ \mC_{\alpha}(\phi)  ] $}}\\  
			\ENDFOR
			\STATE Update each generator $ g_i $ by descending the gradient: $ \nabla_{\theta_i} [{-} \lambda_i^+  \mathop{\mmE}\nolimits_{\hat{\bx}\sim \hat{\mmP}_{\theta_i}} \left[ f \left(\hat{\bx}\right) \right] - \mM_{\alpha}(g_i) ] $
			\ENDWHILE
		\end{small}
	\end{algorithmic}
\end{algorithm}

\subsection{Proposed Objective Function} \label{subsec:mwgan}

\cao{To minimize Wasserstein distance among domains, we now present a novel multi-marginal Wasserstein GAN (MWGAN) based on the proposed dual formulation in (\ref{obj:newD}).} 
Specifically, 
let $ \mF {=} \{ f{:}\, \mmR^d {\to} \mmR \} $ be the class of discriminators parameterized by $ w $, and $ \mG {=} \{ g{:}\, \mmR^d {\to} \mmR^d \}$ be the class of generators and $ g_i {\in} \mG $ is parameterized by $ \theta_i $. 
\cao{Motivated by the adversarial mechanism of WGAN, let $ \lambda_0 {=} 1 $ and $ \lambda_i {:=} {-} \lambda_i^+ $, $\lambda_i^{+}{>}0, i {\in} [N] $}, 
then Problem \ref{problem:newD} can be rewritten as follows: 
\begin{prob} \textbf{\emph{(Multi-marginal Wasserstein GAN)}} \label{problem:MWGAN}
	Given a discriminator $ f{\in}\mF $ and generators $ g_i {\in} \mG, i {\in} [N] $, we can define the following multi-marginal Wasserstein distance as
	\begin{small}
		\begin{align} \label{objective:mwgan}
			W \left(\hat{\mmP}_s, \hat{\mmP}_{\theta_1}, \ldots, \hat{\mmP}_{\theta_N} \right) = 
			\max\limits_{f} \; \mathop{{\mmE}}\nolimits_{\bx \sim \hat{\mmP}_s} \left[ f (\bx) \right] 
			- \sum\nolimits_i \cao{\lambda_i^+} \mathop{{\mmE}}\nolimits_{\hat{\bx} {\sim} \hat{\mmP}_{\theta_i}} \left[ f \left(\hat{\bx}\right) \right], \cao{\quad\st\; \hat{\mmP}_{\theta_i} {\in} \mD_i, f {\in} \Omega}.
		\end{align}
	\end{small}
	\cao{where $ \hat{\mmP}_s $ is the real source distribution, and the distribution $\hat{\mmP}_{\theta_i}$ is generated by $ g_{i} $ in the $ i $-th domain, \begin{small}$ \Omega {=} \{ f | f(\bx) - {\sum_{i{\in}[N]} \lambda_i^+ f(\hat{\bx}^{(i)}) {\leq}  c (\bx, \hat{\bx}^{(1)}, \ldots, \hat{\bx}^{(N)} )}, f{\in}\mF \} $\end{small} with \begin{small}$ \bx {\in} \hat{\mmP}_s $\end{small} and \begin{small}$ \hat{\bx}^{(i)} {\in} \hat{\mmP}_{\theta_i} $\end{small}, $ i {\in} [N] $.} 
\end{prob}
In Problem \ref{problem:MWGAN}, we refer to $ \hat{\mmP}_{\theta_i} {\in} \mD_i, i{\in}[N]$ as \textbf{inner-domain constraints} and $ f {\in} \Omega $ as \textbf{inter-domain constraints} (See Subsections \ref{subsec:inner_domain_constraints} and \ref{subsec:inter_domain_constraints}).
The influence of these constraints are investigated in Section \ref{sec:influence} of supplementary materials.
Note that $ \lambda_i^+ $ reflects the importance of the $ i $-th target domain. 
In practice, we set $ \lambda_i^+ {=} 1/N, i{\in}[N] $ when no prior knowledge is available on the target domains.
To minimize Problem \ref{problem:MWGAN}, we optimize the generators with the following update rule. 

\begin{thm} \label{thm:mwgan_opt}
	If each generator $ g_{i} {\in} \mG, i {\in} [N] $ is locally Lipschitz (see more details of Assumption 1 \cite{arjovsky2017wasserstein}), then there exists a discriminator $ f $ to Problem \ref{problem:MWGAN}, we have the gradient $ \nabla_{\theta_i} W(\hat{\mmP}_s, \hat{\mmP}_{\theta_1}, \ldots, \hat{\mmP}_{\theta_N}) = {-} \lambda_i^{+} \mmE_{{\bx}\sim \hat{\mmP}_s} \left[ \nabla_{\theta_i} f(g_{i}({\bx})) \right] $ for all $ \theta_i, i{\in}[N] $ when all terms are well-defined. 
\end{thm}

Theorem \ref{thm:mwgan_opt} provides a good update rule for optimizing MWGAN. 
Specifically, we first train an optimal discriminator $ f $ and then update each generator along the direction of $ \mmE_{\bx {\sim} \hat{\mmP}_s} \left[ \nabla_{\theta_i} f(g_{i}(\bx)) \right] $.
The detailed algorithm is shown in Algorithm \ref{alg:mwgan}. 
Specifically, 
the generators cooperatively exploit multi-domain correlations (see Section \ref{sec:one_potential} in supplementary materials) and generate samples in the specific target domain to fool the discriminator; the discriminator enforces generated data in target domains to maintain the similar features from the source domain.


\subsection{Inner-domain Constraints} \label{subsec:inner_domain_constraints}
In Problem \ref{problem:MWGAN}, the distribution $ \mmP_{\theta_i} $ generated by the generator $ g_i $ should belong to the $ i $-th domain for any $i$.
To this end, we introduce an auxiliary domain classification loss and the mutual information.  


\textbf{Domain classification loss.}
Given an input $ \bx{:=}\bx^{(0)} $ and generator $ g_i $, we aim to translate the input $ \bx $ to an output $ \hat{\bx}^{(i)} $ which can be classified to the target domain $ \mD_i $ correctly.
To achieve this goal, we introduce an auxiliary classifier $ \phi{:\;} \mX {\to} \mY $ parameterized by $ v $ to optimize the generators. 
Specifically, we label real data {{$ {\bx} {\sim} \hat{\mmP}_{t_i} $}} as $ 1 $, where $ \hat{\mmP}_{t_i} $ is an empirical distribution in the $ i $-th target domain, and we label generated data {{$ \hat{\bx}^{(i)} {\sim} \hat{\mmP}_{\theta_i} $}} as $ 0 $.
Then, the domain classification loss \wrt $ \phi $ can be defined as:
\begin{align}\label{loss:classification}
	\mC_{\alpha}(\phi) = \alpha \cdot {\mmE}_{\bx' {\sim} \hat{\mmP}_{t_i} {\cup} \hat{\mmP}_{\theta_i}} \left[ \ell \left( \phi\left({\bx'} \right), y^{} \right) \right],
\end{align}
where $ \alpha $ is a hyper-parameter, $y$ is corresponding to $\bx'$, and $ \ell(\cdot, \cdot) $ is a binary classification loss, such as hinge loss \cite{zhang2018online}, mean square loss \cite{mao2017least}, cross-entropy loss \cite{goodfellow2014gans} and Wasserstein loss \cite{frogner2015learning}.

\paragraph{Mutual information maximization.}
After learning the classifier $ \phi $, we maximize the lower bound of the mutual information \cite{chen2016infogan, he2017attgan} between the generated image and the corresponding domain, \ie
\begin{small}
	\begin{align}\label{loss:mutual_information}
		\mM_{\alpha}(g_i) = \alpha \cdot {\mmE}_{\bx \sim \hat{\mmP}_s} \left[ \log \phi \left( \left. y^{(i)}{=}1 \right| {g_i(\bx)} \right) \right].
	\end{align}
\end{small}%
By maximizing the mutual information in (\ref{loss:mutual_information}), we correlate the generated image $ g_i(\bx) $ with the $ i $-th domain, and then we are able to translate the source image to the specified domain. 

\subsection{Inter-domain Constraints}\label{subsec:inter_domain_constraints}

Then, we enforce the inter-domain constraints in Problem \ref{problem:MWGAN}, \ie the discriminator $ f {\in} \mF {\cap} \Omega $.
\cao{One can let discriminator be $ 1 $-Lipschitz continuous, but it  may ignore the dependency among domains (see Section  \ref{sec:discuss_lip} in supplementary materials).} 
Thus, we relax the constraints by the following lemma.



\begin{lemma} \textbf{\emph{(Constraints relaxation)}} \label{ieqn:constraint}
	If the cost function $ c(\cdot) $ is measured by $ \ell_2 $ norm, then there exists $ L_f {\ge} 1 $ such that the constraints in Problem \ref{problem:MWGAN} satisfy $ \sum_i {| f(\bx) {-} f (\hat{\bx}^{(i)} ) | }/{ \| \bx {-} \hat{\bx}^{(i)} \| } {\leq} L_f $.
\end{lemma}
Note that $ L_f $ measures the dependency among domains (see Section \ref{sec:proof_lemma_Lf} in supplementary materials).
In practice, $ L_f $ can be calculated with the cost function, or treated as a tuning parameter for simplicity.

\textbf{Inter-domain gradient penalty.}
\cao{In practice, directly enforcing the inequality constraints in Lemma \ref{ieqn:constraint} would have poor performance when generated samples are far from real data. }
We thus propose the following inter-domain gradient penalty.
Specifically, given real data $ \bx $ in the source domain and generated samples $ \hat{\bx}^{(i)} $, 
if $ \hat{\bx}^{(i)} $ can be properly close to $ \bx $, as suggested in \cite{petzka2018on}, we can calculate its gradient and introduce the following regularization term into the objective of MWGAN, \ie
\begin{small}
	\begin{align} \label{loss:gradient_reg}
		\mR_{\tau}(f) = \tau \cdot \left( \sum\nolimits_i \mathop{{\mmE}}\nolimits_{\tilde{\bx}^{(i)} {\sim} \hat{\mmQ}_i} \left\| \nabla f \left(\tilde{\bx}^{(i)} \right) \right\| {-} L_{f} \right)_+^2,
	\end{align}
\end{small}%
where $ (\cdot)_+ {=} \max\{0, \cdot\} $, $ \tau $ is a hyper-parameter, $ \tilde{\bx}^{(i)} $ is sampled between $ \bx $ and $ \hat{\bx}^{(i)} $, and $ \hat{\mmQ}_i, i {\in} [N] $ is a constructed distribution relying on some sampling strategy.
In practice, one can construct a distribution where samples $ \tilde{\bx}^{(i)} $ can be interpolated between real data $ \bx $ and generated data $ \hat{\bx}^{(i)} $ for every domain \cite{Gulrajani2017gangp}.
Note that the gradient penalty captures the dependency of domains since the cost function in Problem \ref{problem:MWGAN} measures the distance among all domains jointly.

\section{Theoretical Analysis}
In this section, we provide the generalization analysis for the proposed method. 
Motivated by \cite{arora2017gans}, we give a new definition of generalization for multiple distributions as follows.

\begin{deftn} \textbf{\emph{(Generalization) }}\label{def:generalization}
	Let $ \mmP_s $ and $ \mmP_{\theta_i} $ be the continuous real and generated distributions, and $ \hat{\mmP}_s $ and $ \hat{\mmP}_{\theta_i} $ be the empirical real and generated distributions.
	The distribution distance $ W(\cdot, \ldots, \cdot) $ is said to generalize with $ n $ training samples and error $ \epsilon $, if for every true generated distribution $ \mmP_{\theta_i} $, the following inequality holds with high probability, 
	\begin{small}
		\begin{align}
			\left|W\left(\hat{\mmP}_s, \hat{\mmP}_{\theta_1}, \ldots, \hat{\mmP}_{\theta_N}\right) - W(\mmP_s, \mmP_{\theta_1}, \ldots, \mmP_{\theta_N}) \right| \leq \epsilon.
		\end{align}
	\end{small}
\end{deftn}
In Definition \ref{def:generalization}, the generalization bound measures the difference between the expected distance and the empirical distance. 
In practice, our goal is to train MWGAN to obtain a small empirical distance, so that the expected distance would also be small.

With the help of Definition \ref{def:generalization}, we are able to analyze the generalization ability of the proposed method.
Let $ \kappa $ be the capacity of the discriminator, and if the discriminator is $ L $-Lipschitz continuous and bounded in $ [-\Delta, \Delta] $, then
we have the following generalization bound.

\begin{thm} \textbf{\emph{(Generalization bound) }} \label{thm:generalization}
	Given the continuous real and generated distributions $ \mmP_s $ and $ \mmP_{\theta_i}, i{\in}\mI $, and the empirical versions $ \hat{\mmP}_s $ and $ \hat{\mmP}_{\theta_i}, i{\in}\mI $ with at least $ n $ samples in each domain, there is a universal constant $ C $ such that $ n {\ge} {C \kappa \Delta^2 \log(L\kappa/\epsilon) }/{\epsilon^2} $ with the error $ \epsilon $, the following generalization bound is satisfied with probability at least $ 1{-}e^{-\kappa} $,
	\begin{small}
		\begin{align}
			\left|W\left(\hat{\mmP}_s, \hat{\mmP}_{\theta_1}, \ldots, \hat{\mmP}_{\theta_N}\right) - W(\mmP_s, \mmP_{\theta_1}, \ldots, \mmP_{\theta_N}) \right| \leq \epsilon.
		\end{align}
	\end{small}
\end{thm}


Theorem \ref{thm:generalization} shows that MWGAN has a good generalization ability with enough training data in each domain.
In practice, if successfully minimizing the multi-domain Wasserstein distance \ie $ W(\hat{\mmP}_s, \hat{\mmP}_{\theta_1}, \ldots, \hat{\mmP}_{\theta_N}) $, the expected distance $ W(\mmP_s, \mmP_{\theta_1}, \ldots, \mmP_{\theta_N}) $ can also be  small.

%

\section{Experiments}

\textbf{Implementation details.}
All experiments are conducted based on PyTorch, with an NVIDIA TITAN X GPU.\footnote{The source code of our method is available: https://github.com/caojiezhang/MWGAN.}
We use Adam  \cite{kingma2014adam} with $\beta_1{=}0.5$ and $\beta_2{=}0.999$ and set the learning rate as 0.0001. 
We train the model 100k iterations with batch size 16. 
We set $\alpha{=}10$, $ \tau {=} 10 $ and $ L_f $ to be the number of target domains in Loss (\ref{loss:gradient_reg}).
The details of the loss function and the network architectures of the discriminator, generators and classifier can be referred to Section \ref{sec:net_architecture} in supplementary materials.

\begin{figure*}[t]
	\centering
	{
		\includegraphics[width=0.88\linewidth]{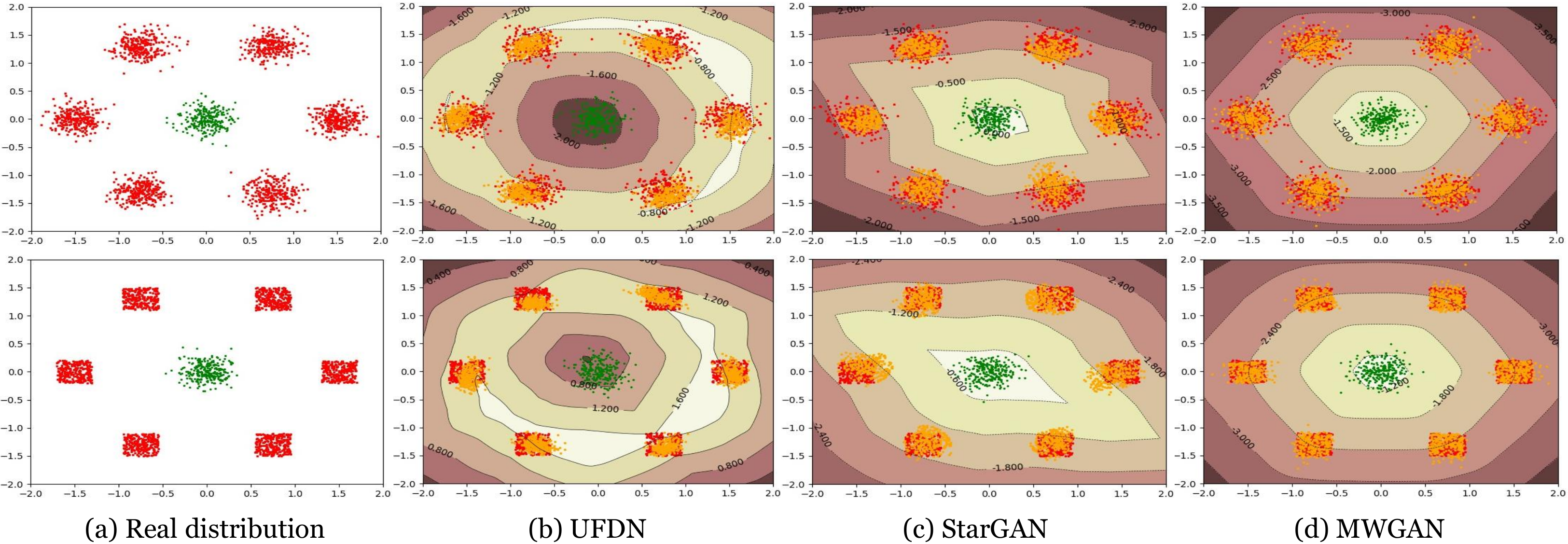} 
		\caption{{Comparisons of distribution matching abilities on the value surface of discriminator. 
				Each method learns from a Gaussian distribution to other six Gaussian (upper line) or Uniform distributions (lower line). (Green: source distribution; Red: target distributions; Orange: generated distributions. ) }
			\label{fig:comparison_contour}
	}}
\end{figure*}

\jie{
	\textbf{Baselines.}
	We adopt the following methods as baselines:
	\textbf{(i) CycleGAN} \cite{zhu2017unpaired} is a two-domain image translation method which can be flexibly extended to perform the multi-domain image translation task. 
	\textbf{(ii) UFDN} \cite{liu2018ufdn} and \textbf{(iii) StarGAN} \cite{choi2018stargan} are multi-domain image translation methods.
}

\textbf{Datasets.}
We conduct experiments on three datasets. Note that all images are resized as 128$ \times $128. 
\textbf{(i) Toy dataset.}
We generate a Gaussian distribution in the source domain, and other six Gaussian or Uniform distributions in the target domains.
More details can be found in the supplemental materials.\\
\textbf{(ii) CelebA \cite{liu2015celeba}}
contains 202,599 face images, where each image has 40 binary attributes. 
We use the following attributes: hair color (black, blond and brown), eyeglasses, mustache and pale skin.
In the first experiment, we \jie{use black hair images as the source domain}, and use the blond hair, eyeglasses, mustache and pale skin images as target domains.
In the second experiment, \mo{we extract 50k Canny edges from CelebA}. We take edge images as the source domain and hair images as target domains.\\ 
\textbf{(iii) Style painting \cite{zhu2017unpaired}.}
The size of \mo{Real scene, }Monet, Van Gogh and Ukiyo-e is 6287, 1073, 400 and 563, respectively. 
We take real scene images as the source domain, and others as target domains. 


\textbf{Evaluation Metrics.} We use the following evaluation metrics:
\textbf{(i) Frechet Inception Distance (FID) \cite{heusel2017gans}} evaluates the quality of the translated images. 
In general, a lower FID score means better performance.
\textbf{(ii) Classification accuracy} widely used in \cite{choi2018stargan, he2017attgan} evaluates the probability that the generated images belong to corresponding target domains.
Specifically, we train a classifier on CelebA (90\% for training and 10\% for testing) using ResNet-18 \cite{He_2016_CVPR}, resulting in a near-perfect accuracy, then use the classifier to measure the classification accuracy of the generated images. 


\begin{figure*}[t]
	\centering
	{
		\includegraphics[width=0.93\linewidth]{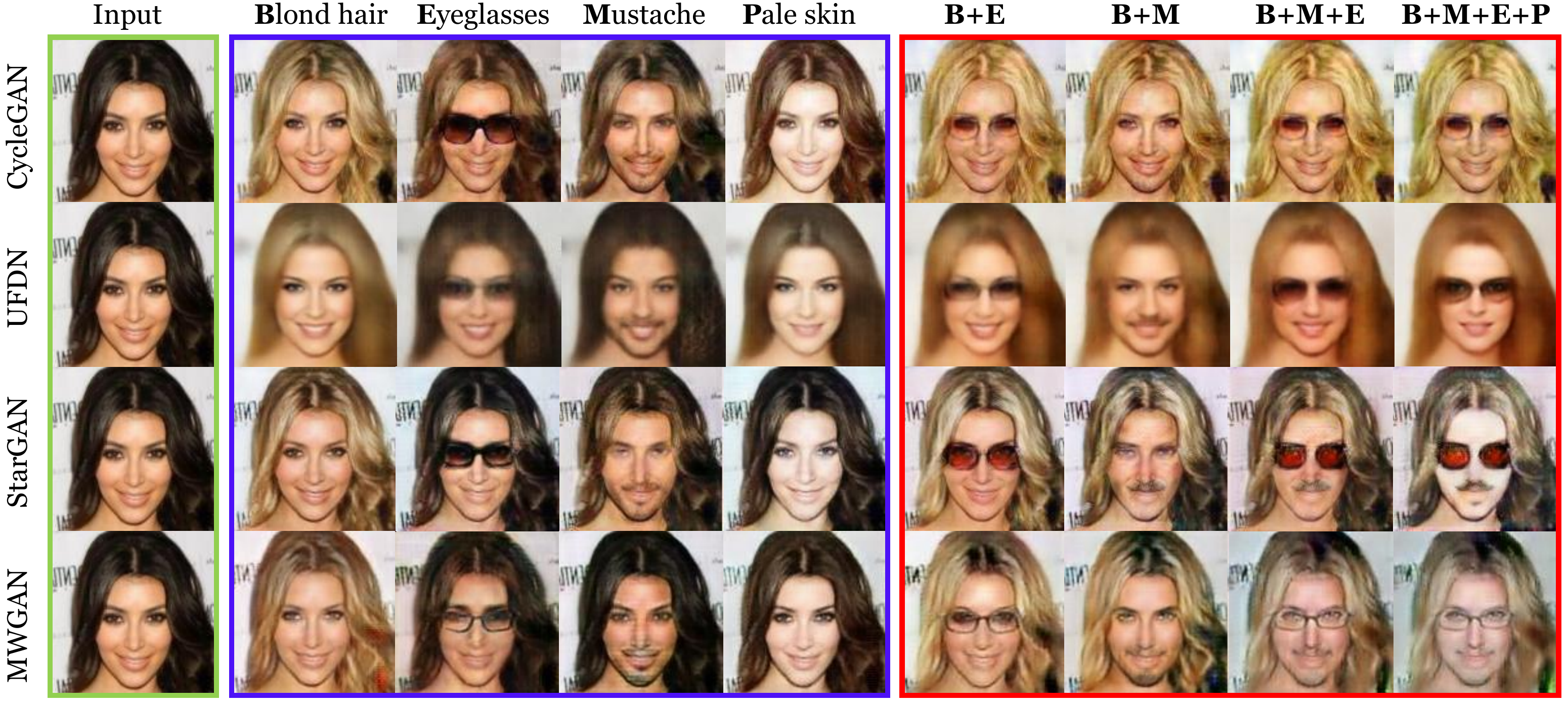}
		\caption{{Comparisons of attribute translation on CelebA. The first column shows the input images, the next four columns show the single attribute translation results, and the last four columns show the multi-attribute translation results. (B: Blond hair; E: Eyeglasses; M: Mustache; P: Pale skin.)}} 
		\label{fig:celeba_5domain}
	}
\end{figure*}


\subsection{Results on Toy Dataset}
We compare MWGAN with UFDN and StarGAN on toy dataset to 
verify the limitations mentioned in Section \ref{sec:related_work}. 
Specifically, we measure the distribution matching ability and plot the value surface of the discriminator. 
Here, the value surface depicts the outputs of the discriminator \cite{Gulrajani2017gangp, liu2018two}.

In Figure \ref{fig:comparison_contour}, MWGAN matches the target domain distributions very well as it is able to capture the geometric information of real distribution using a low-capacity network. 
Moreover, the value surface shows that the discriminator provides correct gradients to update the generators.
However, the baseline methods are very sensitive to the type of source and target domain distributions. 
With the same capacity, the baseline methods on similar distributions (top row) are able to match the target domain distributions. 
\mo{However}, they cannot match the target domain distribution well when the initial and the target domain distributions are different (see bottom row of Figure \ref{fig:comparison_contour}). 


\begin{table}[!t]
	\caption{Comparisons of FID and classification accuracy (\%) on single facial attribute translation.}	
	\label{table:single_att_fid_acc}
	\centering
	\resizebox{1\textwidth}{!}{
		\begin{tabular}{c|c|c|c|c|c|c|c|c}
			\hline
			\multirow{2}[0]{*}{Method} 
			& \multicolumn{2}{c|}{Hair} & \multicolumn{2}{c|}{Eyeglasses} & \multicolumn{2}{c|}{Mustache} & \multicolumn{2}{c}{Pale skin} \\
			\cline{2-9}
			& FID & Accuracy (\%) & FID & Accuracy (\%) & FID & Accuracy (\%) & FID & Accuracy (\%) \\
			\hline
			CycleGAN  & 20.45 & 95.07 & 23.69 & 96.94 & 24.94 & 93.89 & 18.09 & 80.75 \\	
			UFDN      & 65.06 & 92.01 & 69.30 & 79.34 & 76.04 & 97.18 & 53.11 & 83.33 \\
			StarGAN   & 23.47 & 96.00 & 25.36 & 99.51 & 23.75 & \textbf{99.06} & 18.12 & 92.48 \\
			\hline
			MWGAN     & \textbf{19.63} & \textbf{97.65} & \textbf{22.94} & \textbf{99.53} & \textbf{23.69} & 98.35 & \textbf{15.91} & \textbf{93.66}  \\
			\hline
		\end{tabular}
	} 
\end{table}

\begin{table}[t!]
	\begin{minipage}{0.49\textwidth}
		\centering
		\makeatletter\def\@captype{table}\makeatother\caption{Comparisons of classification accuracy (\%) on multi-attribute synthesis. (B: Blond hair, E: Eyeglasses, M: Mustache, P: Pale skin.)}
		\label{table:multi_att_fid_acc}
		\resizebox{1\textwidth}{!}{
			\begin{tabular}{c|c|c|c|c}
				\hline
				{Method}  & {B+E} & {B+M} & {B+M+E} & {B+M+E+P} \\
				\hline
				CycleGAN  & 66.43 & 33.33 & 11.03 & 2.11  \\	
				UFDN      & 72.53 & 51.40 & 23.00 & 8.54  \\
				StarGAN   & 66.66 & 62.20 & 45.77 & 6.10  \\
				\hline
				MWGAN     & \textbf{75.82} & \textbf{69.01} & \textbf{53.75} & \textbf{19.95}  \\
				\hline
			\end{tabular}
		}
	\end{minipage}
	~~
	\begin{minipage}{0.49\textwidth}
		\centering
		\makeatletter\def\@captype{table}\makeatother\caption{Comparisons of the FID value for each facial attribute (different colors of hair) on the Edge$ \rightarrow $CelebA translation task.}
		\label{table:edge2hair}
		\resizebox{1\textwidth}{!}{
			\begin{tabular}{c|c|c|c}
				\hline
				{Method} & {Black hair} & {Blond hair} & {Brown hair} \\
				\hline
				CycleGAN &  65.10 &  81.59 &65.79  \\	
				UFDN     & 131.65 & 144.78 & 88.40  \\
				StarGAN  &  53.41 &  81.00 & 57.51  \\
				\hline
				MWGAN    & \textbf{33.81} & \textbf{51.87} & \textbf{35.24}  \\
				\hline
			\end{tabular}
		}
	\end{minipage}
\end{table}

\subsection{Results on CelebA} \label{subsec:res_celeba}
We compare MWGAN with several baselines on both balanced and imbalanced translation tasks.

\textbf{\emph{(i) Balanced image translation task. }}
In this experiment, 
we train the generators to produce single attribute images, and then synthesize multi-attribute images using the composite generators. 
We generate attributes in order of \{Blond hair, Eyeglasses, Mustache, Pale skin\}. Taking two attributes as an example, let $ g_1 $ and $ g_2 $ be the generators of Blond hair and Eyeglasses images, respectively, then images with Blond hair and Eyeglasses attributes are generated by the composite generators $ g_2 {\circ} g_1 $. 

\textbf{Qualitative results. }
In Figure~\ref{fig:celeba_5domain}, MWGAN has a better or comparable performance than baselines on the single attribute translation task, but achieves the highest visual quality of multi-attributes translation results.
In other words, MWGAN has good generalization performance. 
However, CycleGAN is hard to synthesize multi-attributes. 
UFDN cannot guarantee the identity of the translated images and produces images with blurring structures. 
Moreover, StarGAN highly depends on the number of \mo{transferred} domains and the synthesized images sometimes lack the perceptual realism.

\textbf{Quantitative results.}
We further compare FID and classification accuracy for the single-attribute results.  
For the multi-attribute results, we only report classification accuracy because FID is no longer a valid measure and may give misleading results when training data are not sufficient \cite{heusel2017gans}.
In Table \ref{table:single_att_fid_acc}, MWGAN achieves the lowest FID and comparable classification accuracy, indicating that it produces realistic single-attribute images of the highest quality. 
In Table \ref{table:multi_att_fid_acc}, MWGAN achieves the highest classification accuracy and thus synthesizes the most realistic multi-attribute images.


\emph{\textbf{(ii) Imbalanced image translation task.}}
In this experiment, we compare MWGAN with baselines on the Edge$ {\rightarrow} $CelebA translation task. 
Note that this task is unbalanced because the information of edge images is much less than facial attribute images.

\textbf{Qualitative results.}
In Figure \ref{fig:edge2celeba}, MWGAN is able to generate the most natural-looking facial images with the corresponding attributes from edge images.
In contrast, 
UFDN fails to preserve the facial texture of an edge image, and generates images with very blurry and distorted structure. 
In addition, CycleGAN and StarGAN mostly preserve the domain information but cannot maintain the sharpness of images and the facial structure information. 
%
Moreover, this experiment also shows the superiority of our method on the imbalanced image translation task.


\textbf{Quantitative results.}
In Table \ref{table:edge2hair}, 
MWGAN achieves the lowest FID, showing that it is able to produce the most realistic facial attributes from the edge images.
In contrast, the FID values of baselines are large because these methods are hard to generate sharp and realistic images. 
We also perform a perceptual evaluation with AMT for this task(see Section \ref{sec:amt} in supplementary materials).

\begin{figure}[t]
	\begin{minipage}[t]{0.48\linewidth}
		\centering
		\includegraphics[width = 0.95\columnwidth]{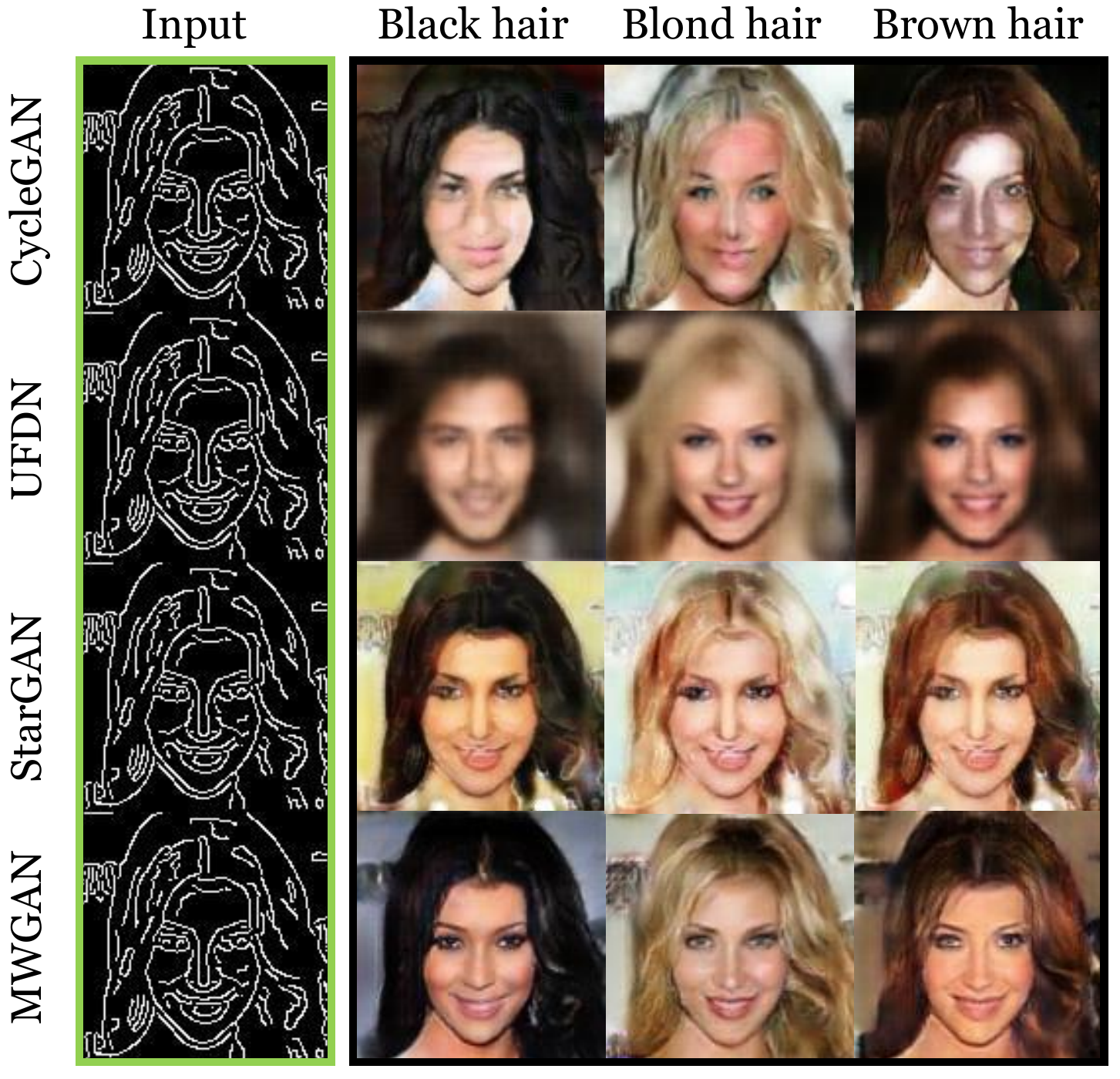}
		\caption{Comparisons of the edge${\rightarrow}$CelebA translation results. The first column shows the input images, and the next three columns show the single attribute translation results.}  
		\label{fig:edge2celeba}
	\end{minipage}
	~~~
	\begin{minipage}[t]{0.48\linewidth}
		\centering
		\includegraphics[width = 0.95\columnwidth]{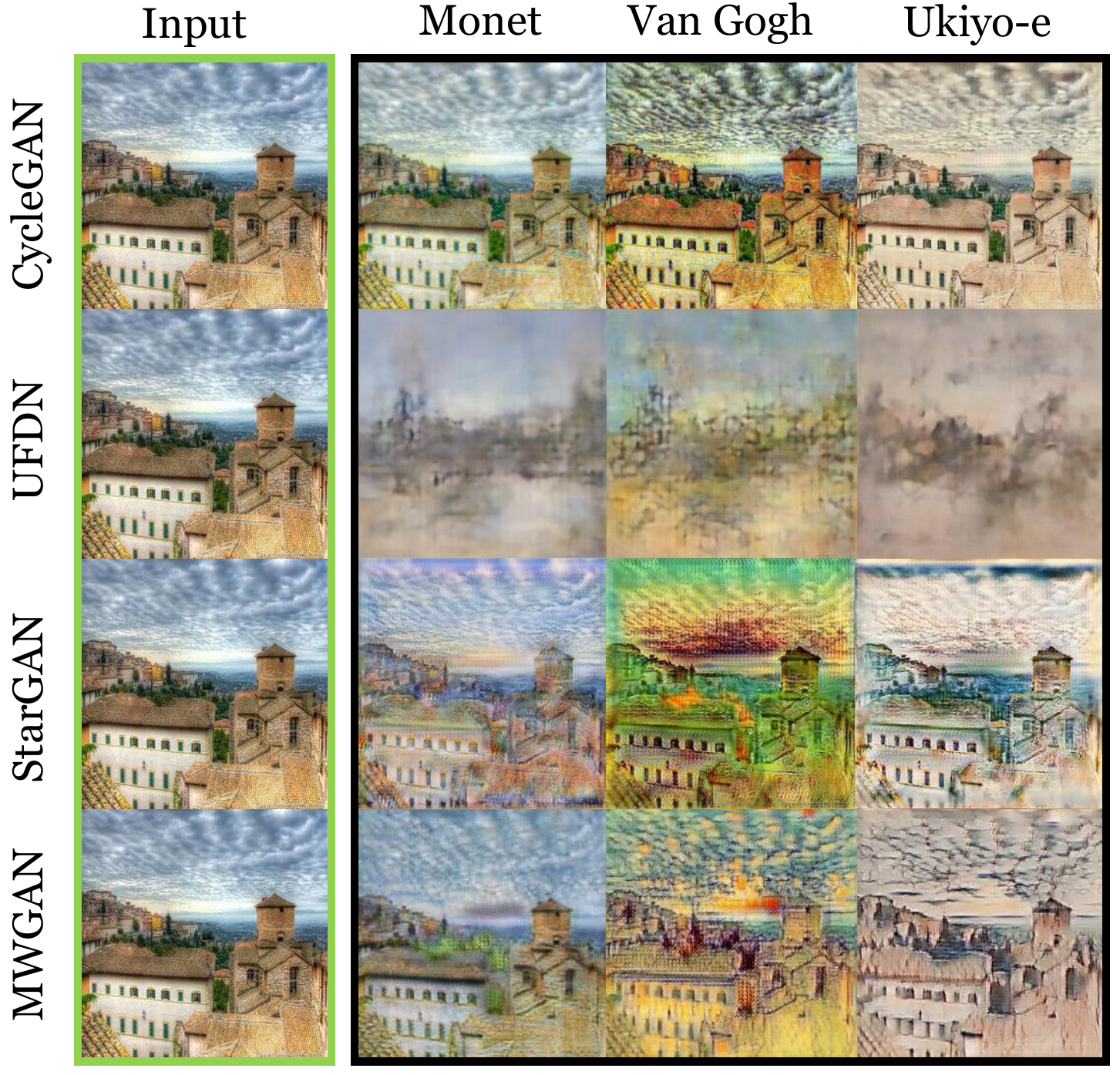}
		\caption{Comparisons of style transfer results. The first column shows the real world images, and the last three columns show translation results, \ie Monet, Van Gogh and Ukiyo-e.}
		\label{fig:photo2art}
	\end{minipage}
\end{figure}

\subsection{Results on Painting Translation}
In this experiment, we finally train our model on the painting dataset to conduct the style transfer task.
As suggested in \cite{gatys2016image, gatys2017controlling, zhu2017unpaired}, we only show the qualitative results.
Note that this translation task is also imbalanced because the input and target distributions are significantly different. 
In Figure \ref{fig:photo2art}, MWGAN generates painting images with higher visual quality. 
In contrast, UFDN fails to generate clearly structural painting images because it is hard to learn domain-invariant representation \mo{when domains are highly imbalanced.} 
CycleGAN cannot fully learn some useful information from painting images to scene images.
When taking a painting image as an input, StarGAN may obtain misleading information to update the generator. 
In this sense, when all domains are significantly different, StarGAN may not learn a good single generator to synthesize  images of multiple domains.

\section{Conclusion}

In this paper, we have proposed a novel multi-marginal Wasserstein GAN (MWGAN) for multiple marginal matching problem. 
Specifically, with the help of multi-marginal optimal transport theory, we develop a new dual formulation for better adversarial learning on the unsupervised multi-domain image translation task. 
Moreover, we theoretically define and further analyze the generalization ability of the proposed method. 
Extensive experiments on both toy and real-world datasets demonstrate the effectiveness of the proposed method. 

\subsubsection*{Acknowledgements}
This work is partially funded by Guangdong Provincial Scientific and Technological Funds under Grants 2018B010107001, National Natural Science Foundation of China (NSFC) 61602185, key project of NSFC (No. 61836003), Fundamental Research Funds for the Central Universities D2191240, Program for Guangdong Introducing Innovative and Enterpreneurial Teams 2017ZT07X183, and 
Tencent AI Lab Rhino-Bird Focused Research Program (No. JR201902).
This work is also partially funded by Microsoft Research Asia (MSRA Collaborative Research Program 2019).


\bibliographystyle{abbrv} 


\newpage
\appendix

\def\blfootnote{\let\thefootnote\relax\footnotetext}

\begin{table}
	\setlength{\tabcolsep}{0.2cm}
	\begin{tabular}{p{0.97\columnwidth}}
		\nipstophline
		\rule{0pt}{1.0cm}
		\centering
		\Large{\textbf{Supplementary Materials: Multi-marginal Wasserstein GAN}}
		\vspace{4pt}
	\end{tabular}
\end{table}
\begin{table}
	\begin{tabular}{c}
		\nipsbottomhline
		~~~~~~~~~~~~~~~~~~~~~~~~~~~~~~~~~~~~~~~~~~~~~~~~~~~~~~~~~~~~~~~~~~~~~~~~~~~~~~~~~~~~~~~~~~~~~~~~~~~~~~~~~~~~~~~~~~~~~~~~~~~~~~~~~~~~~~~~~~~~~~~~~~~~~~~~~~
	\end{tabular}
\end{table}

{\centering{
		\textbf{Jiezhang Cao$^{1*}$\blfootnote{$^*$Authors contributed equally.}, Langyuan Mo$^{1*}$, Yifan Zhang$^{1}$, Kui Jia$^{1}$,  Chunhua Shen$^3$, Mingkui Tan$^{1,2 * \dag}$}\blfootnote{$^\dag$Corresponding author.} \\
		$^1$South China University of Technology, $^2$Peng Cheng Laboratory, $^3$The University of Adelaide \\
		\{secaojiezhang, selymo, sezyifan\}@mail.scut.edu.cn \\ 
		~~~~~~~~~~~~~~~~~~~~~~~~~\{mingkuitan, kuijia\}@scut.edu.cn, chunhua.shen@adelaide.edu.au
}}

\setcounter{section}{0}
\renewcommand\thesection{\Alph{section}}
~\\
\paragraph{Organization.}
In the supplementary materials, we provide detailed proofs for all theorems, lemmas and propositions of our paper \cite{ref_cao2019mwgan}, and more experiment settings and results.
We organize our supplementary materials as follows.

\textbf{Theory part.}
In Section \ref{sec:pre_mmot}, we provide preliminaries of multi-marginal optimal transport.
In Section \ref{sec:equiv_thm}, we prove an equivalence theorem that solving Problem \ref{problem:dist_D} is equivalent to solving Problem \ref{problem:newD} under a mild assumption.
In Section \ref{sec:solution_exist}, we build the relationship between Problem \ref{problem:dist_D} and Problem \ref{problem:newD}.
In Section \ref{sec:error_bound}, we provide an error bound of the new dual formulation.
In Section \ref{sec:mwgan_opt}, we prove an update rule of optimizing the generators in Problem \ref{problem:MWGAN}.
In Section \ref{sec:generalization}, we theoretically analyze the generalization performance of MWGAN.
In Section \ref{sec:proof_lemma_Lf}, we provide a relaxation of inter-domain constraints.
In Section \ref{sec:discuss_lip}, we discuss a case that the potential function is Lipschitz continuous.

\textbf{Experiment part.}
In Section \ref{sec:diff_gan}, we compare MWGAN with existing GAN methods.
In Section \ref{sec:one_potential}, we study the effectiveness of one potential function of MWGAN.
In Section \ref{sec:toy_dataset}, we introduce more details of toy dataset.
In Section \ref{sec:details_classification_CelebA}, we introduce details of the classification on CelebA.
In Section \ref{sec:amt}, we apply more quantitative evaluations for MWGAN.
In Section \ref{sec:influence}, we discuss the influences of inner-domain constraints and inter-domain constraints.
In Section \ref{sec:influence_Lf}, we discuss the influence of the hyper-parameter in our proposed method.
In Section \ref{sec:net_architecture}, we introduce the details about the network architecture of the discriminator and generators as well as more training details of MWGAN.
In Section \ref{sec:more_qualitative_results}, we present more qualitative results on CelebA and style painting dataset.

\section{Preliminaries of Multi-marginal Optimal Transport} \label{sec:pre_mmot} 
\paragraph{Notation.}
We use calligraphic letters (\textit{e.g.}, $ \mX $) to denote space, capital letters (\textit{e.g.}, $ X $) to denote random variables, and bold lower case letter (\textit{e.g.}, $ \bx $) to denote the corresponding values. 
Let $ \mD {=} (\mX , \mmP) $ be the domain, $ \mmP $ or $ \mu $ be the marginal distribution over $ \mX $ and $ \mP(\mX) $ be the set of all the probability measures over $ \mX $.
For convenience, let $ \mX {=} \mmR^d $, and let $ \mI {=} \{ 0, ..., N \} $ and $ [N]{=}\{1, ..., N\} $.%

Deep learning has achieved great success in computer vision.
Despite its empirical success, however, the theoretical understanding of deep neural networks still remains an open problem.
Existing analysis methods \cite{ref_cao2017flatness} are hard to understand the deep neural networks.
Recently, optimal transport \cite{ref_villani2008optimal, ref_cao2019learning, ref_yan2019oversampling} has been applied in deep learning \cite{ref_zhang2019miccai}.
With the help of optimal transport theory, one can define the following primal problem to measure the distance among all distributions jointly (see Figure \ref{fig:comparison}).
Specifically, the primal formulation of the multi-marginal Kantorovich problem is defined as follows.

\begin{prob} \textbf{ \emph{(Primal problem \cite{ref_santambrogio2015optimal})}}  \label{problem:MK-P}
	Given $ N+1 $ marginals $ \mu_i \in \mP(\mmR^d),\; \forall\; i {\in} \mI $ and a cost function $ c\left(X^{(0)}, \ldots, X^{(N)}\right) $, then the multi-marginal Kantorovich problem can be defined as:
	\begin{align}
		\inf_{{\gamma} {\in} \Pi(\mu_0{,} \ldots{,} \mu_N)}  \int {c} \left(X^{(0)}, \ldots, X^{(N)} \right) d \gamma \left(X^{(0)}, \ldots, X^{(N)} \right){,}
	\end{align}
	where $ \Pi (\mu_0, \ldots, \mu_N) $ is the set of probabilistic couplings $ \gamma \left(X^{(0)}, \ldots, X^{(N)} \right) $ with the marginal $ \mu_i $, for all $ i \in \mI $, $\Pi (\mu_0, \ldots, \mu_N) {:=} \left\{ \gamma \in \mP\left(\mmR^{d (N{+}1)}\right) | \pi_i(\gamma) {=} \mu_i, \forall i \in \mI \right\}$, where $ \pi_i: \mmR^{d (N{+}1)} {\to} \mmR^d $ is the canonical projection.
\end{prob}
Solving the primal problem is intractable on the generative task \cite{ref_genevay2018learning}, so we consider the dual formulation of the multi-marginal Kantorovich problem. 
\begin{prob} \textbf{\emph{(Dual problem \cite{ref_santambrogio2015optimal})}  }
	Given $ N{+}1 $ marginals $ \mu_i {\in} \mP(\mmR^d) $ and potentials $ f_i, i {\in} \mI $, the dual Kantorovich problem of multi-marginal Wasserstein distance is defined as:
	\begin{equation}
		\begin{aligned}
			W(\mu_0, \ldots, \mu_N) {=} 
			&\sup_{f_i} \sum\limits_{i \in \mI} \int  f_i \left(X^{(i)}\right) d \mu_i \left(X^{(i)}\right), \quad\\
			&~\st \sum\nolimits_{i \in \mI} f_i \left(X^{(i)}\right) {\leq} c \left(X^{(0)}, \ldots, X^{(N)} \right).
		\end{aligned}
	\end{equation}
\end{prob}

\cao{In practice, we optimize the discrete case of Problem \ref{problem:MK-D}.
	Specifically, given samples \begin{small}{$ \{\bx_j^{(0)}\}_{j{\in}\mJ_0} $} \end{small} and \begin{small}$ \{\bx_j^{(i)}\}_{j{\in}\mJ_i} $\end{small} drawn from source domain distribution $ \mmP_s $ and generated target distributions \begin{small}$ \mmP_{\theta_i}, i {\in} [N] $\end{small}, respectively, where $ \mJ_i $ is an index set and $ n_i {=} |\mJ_i| $ is the size of samples, we have: }
\begin{prob}\textbf{\emph{(Discrete dual problem)}}
	Let $ F {=} \{f_0, \ldots, f_N\} $ be the set of Kantorovich potentials, then the discrete dual problem $\hat{h} (F)$ can be defined as:
	\begin{align}
		&\max_F \; \hat{h} (F) = \sum\nolimits_{i} \frac{1}{n_i} \sum\nolimits_{j \in \mJ_i} f_i\left(\bx_{j}^{(i)} \right), \\
		&~\st \sum\nolimits_{i} f_i \left(\bx^{(i)}_{k_i} \right) \leq  c \left(\bx^{(0)}_{k_0}, \ldots, \bx^{(N)}_{k_N} \right), \forall k_i \in [n_i]. 
	\end{align}
\end{prob}

There is an interesting class of functions satisfying the constraint in Problem \ref{problem:MK-D}, it is helpful for deriving Theorem \ref{thm:pd_opt}.
\begin{deftn} \textbf{\emph{($ c $-conjugate function) }} \label{def:conjugate}
	Let $ c: \mmR^{d(N{+}1)} \to \mmR \cup \{\infty\} $ be a Borel function. We say that the $ (N{+}1) $-tuple of functions $ (f_0, \ldots, f_N) $ is a $ c $-conjugate function, $ \forall i \in \mI $, if 
	\begin{align}
		f_i \left( X^{(i)} \right) = \inf \left\{ c \left(X^{(0)}, \ldots, X^{(N)} \right) - \sum\nolimits_{j\neq i}^{N} f_j \left(X^{(j)} \right) \right\}.
	\end{align}
\end{deftn}
With Definition \ref{def:conjugate}, the following theorem builds a relationship between the primal and dual problem.
\begin{thm} \textbf{\emph{(Primal-dual Optimality \cite{ref_kellerer1984duality})}}  \label{thm:pd_opt}
	Let $ (\mmR^d, \mu_0), \ldots, (\mmR^d, \mu_N) $ be Polish spaces equipped with Borel probability measures $ \mu_0, \ldots, \mu_N $, then we have  
	\begin{enumerate}[leftmargin=*] 
		\item There exists a solution $ \gamma $ to Problem \ref{problem:MK-P} and a c-conjugate solution $ (f_0, \ldots, f_N) $ to Problem \ref{problem:MK-D}. 
		\item The maximum value of Problem \ref{problem:MK-P} is equal to the minimum value of Problem \ref{problem:MK-D}.
		\item For any solution $ \gamma $ of Problem \ref{problem:MK-P}, any c-conjugate solution of Problem \ref{problem:MK-D} and any $ (X^{(0)}, {\ldots}, X^{(N)}) $ in the support of $ \gamma $, then 
		\begin{align*}
			\sum_{i=0}^{N} f_i ( X^{(i)} ) = c (X^{(0)}, {\ldots}, X^{(N)} ).
		\end{align*}
	\end{enumerate}
\end{thm}
This Primal-dual Optimality theorem helps to derive a new dual formulation in Theorem \ref{thm:solution_exist}.

\newpage
\section{Equivalent Theorem} \label{sec:equiv_thm}

The Kantorovich duality theorem \cite{ref_santambrogio2015optimal} can transform Problem \ref{problem:dist_D} into the following problem.
\begin{prob}
	Let $ F_c = \{f_1, \ldots, f_N\} $ be the set of Kantorovich potentials, then the discrete $ c $-conjugate dual problem can be defined as:
	\begin{align}
		\sup_{F_c} \; \hat{h} (F_c) = \frac{1}{n_0} \sum\limits_{j \in \mJ_0} f^c\left(\bx_{j}^{(0)} \right) + \sum\limits_{i=1}^N \frac{1}{n_i} \sum\limits_{j \in \mJ_i} f_i\left(\bx_{j}^{(i)} \right),  
	\end{align}
	where $ f^c $ is the $ c $-conjugate function defined as:
	\begin{align}\label{eqn:c_conjugate}
		f^c \left( \bx^{(0)} \right) = \inf_{\bx^{(1)}, \ldots, \bx^{(N)}} \left\{ c \left(\bx^{(0)}, \bx^{(1)}, \ldots, \bx^{(N)} \right) - \sum\limits_{j=1}^{N} f_j \left(\bx^{(j)} \right) \right\}.
	\end{align}
\end{prob}

\begin{deftn} \textbf{\emph{(Cost function)}}  \label{def:cost_triangle_ineqn}
	Given a distance function $ d(\cdot, \cdot) $ which satisfies the triangle inequality, \ie $ d(\bx, \by) {+} d(\by, \bz) {\geq} d(\bx, \bz), \forall \bx, \by, \bz $, then the cost function can be defined as 
	\begin{align}
		c\left( \bx^{(0)}, \bx^{(1)}, \ldots, \bx^{(N)} \right) = \sum\nolimits_{i \neq j} d\left(\bx^{(i)}, \bx^{(j)}\right), \quad \forall \; i,j \in [N].
	\end{align}
\end{deftn}

\begin{lemma} \label{lemma:problem_II}
	Given the cost function $ c(\cdot, \ldots, \cdot) $ defined in Definition \ref{def:cost_triangle_ineqn}, if $ N{+}1 $ samples $ \bx_{k_i}^{(i)} \in \mX^{(i)}, i \in \mI $ are overlapped, let $ f_i^*, i {\in} [N] $ be the optimizers to Problem \ref{problem:dist_D}, and $ (f^c)^* $ be the $ c $-conjugate function defined in Eqn. (\ref{eqn:c_conjugate}), then 
	\begin{align}
		(f^c)^*\left(\bx_{k_0}^{(0)}\right) {=} f_i^*\left(\bx_{k_i}^{(i)}\right),\quad i {\in} [N].
	\end{align}
\end{lemma}

\begin{coll} \label{lemma:problem_II_pro}
	Given the cost function $ c(\cdot, \ldots, \cdot) $ defined in Definition \ref{def:cost_triangle_ineqn}, we assume $ f^* $ is $ 1 $-Lipschitz continuous, the samples are bounded and the distance function satisfies $ d(\bx, \bz) {\leq} d(\bx, \by) {+} d(\by, \bz) {\leq} C d(\bx, \bz) $, where $ C{\ge} 1 $, when $ N{+}1 $ samples $ \bx_{k_i}^{(i)} {\in} \mX^{(i)}, i {\in} \mI $ are close to each other, \ie $ c(\bx_{k_0}^{(0)}, \ldots, \bx_{k_N}^{(N)}) $ is arbitrarily small, then $ (f^c)^*(\bx_{k_0}^{(0)}) $ would be arbitrarily close to $ f_i^*(\bx_{k_i}^{(i)}), i {\in} [N] $, where $ f_i^*, i {\in} [N] $ are the optimizers to Problem \ref{problem:dist_D}, and $ (f^c)^* $ is the $ c $-conjugate function defined in Eqn. (\ref{eqn:c_conjugate}).
\end{coll}

\begin{lemma}\label{lemma:problem_III}
	Suppose $ f^* $ is an optimal solution to Problem \ref{problem:newD} and $ \sum_{i\in\mI} \lambda_i = 0 $, then $ f^* $ satisfies 
	\begin{align}
		f^*\left(\bx^{(0)}\right) = \inf_{\bx^{(1)}, \ldots, \bx^{(N)}} \left\{c \left(\bx^{(0)}, \bx^{(1)}, \ldots, \bx^{(N)} \right) - \sum\nolimits_{j=1}^{N} \lambda_j f^* \left(\bx^{(j)} \right) \right\}, \quad \forall \; \bx \in \mX^0.
	\end{align}
\end{lemma}

\begin{thm} \textbf{\emph{(Equivalent Theorem)}} \label{thm:equivalent_theorem}
	Given the cost function defined in Definition \ref{def:cost_triangle_ineqn}, and $ \sum_i \lambda_i {=} 0, i{\in}\mI $ then solving Problem \ref{problem:newD} is equivalent to solving Problem \ref{problem:dist_D}, \ie the optimal objective of Problems \ref{problem:dist_D} and \ref{problem:newD} are equal.
\end{thm}

\newpage
\subsection{Proofs of Equivalent Theorem}
\textbf{Theorem \ref{thm:equivalent_theorem}} \textbf{(Equivalent Theorem)}
\emph{
	Given the cost function defined in Definition \ref{def:cost_triangle_ineqn}, and $ \sum_i \lambda_i {=} 0, i{\in}\mI $ then solving Problem \ref{problem:newD} is equivalent to solving Problem \ref{problem:dist_D}, \ie the optimal objective of Problems \ref{problem:dist_D} and \ref{problem:newD} are equal.
}

\begin{proof}
	First, we prove that any optimal solution to Problem \ref{problem:newD} is a feasible solution to Problem \ref{problem:dist_D}.
	Suppose that $ f^* $ is the optimal solution to Problem \ref{problem:newD}, from Lemma \ref{lemma:problem_II}, we know that 
	\begin{align}
		f^*\left(\bx^{(0)}\right) = \inf_{\bx^{(1)}, \ldots, \bx^{(N)}} \left\{c \left(\bx^{(0)}, \bx^{(1)}, \ldots, \bx^{(N)} \right) - \sum\limits_{j=1}^{N} \lambda_j f^* \left(\bx^{(j)} \right) \right\}, \quad \forall \; \bx {\in} \mX^0.
	\end{align}
	From the definition of $ c $-conjugate function in Eqn. (\ref{eqn:c_conjugate}), we have
	\begin{align}
		(f^*)^c\left(\bx^{(0)}\right) = \inf_{\bx^{(1)}, \ldots, \bx^{(N)}} \left\{c \left(\bx^{(0)}, \bx^{(1)}, \ldots, \bx^{(N)} \right) - \sum\limits_{j=1}^{N} \lambda_j f^* \left(\bx^{(j)} \right) \right\}, \quad \forall \; \bx {\in} \mX^0.
	\end{align}
	Hence, $ f^* $ is a feasible solution to Problem \ref{problem:dist_D}.
	Therefore, 
	\begin{align}
		\hat{h}(F_c^*) \geq \hat{h}(F_{\lambda}^*). 
	\end{align}
	Second, we prove that any optimal solution to Problem \ref{problem:dist_D} is a feasible solution to Problem \ref{problem:newD}.
	Suppose $ f_i^*, i\in[N] $ are optimizers to Problem \ref{problem:dist_D}.
	From Lemma \ref{lemma:problem_II}, $ \forall\; \bx_{k_i}^{(i)} \in \mX^{(i)}, i \in \mI $ with equal value, we have $ (f^c)^*(\bx_{k_0}^{(0)}) = f_i^*(\bx_{k_i}^{(i)}), i \in [N] $, given that the cost function satisfies the condition in Definition \ref{def:cost_triangle_ineqn}.
	Therefore, we can find a function $ \phi $ and $ \lambda_i, i\in[N] $ such that 
	\begin{align}
		\lambda_0 \phi\left(\bx_{k_0}^{(0)}\right) &= (f^c)^*\left(\bx_{k_0}^{(0)}\right), \\ 
		\lambda_i \phi \left( \bx_{k_i}^{(i)} \right) &= f_i^*\left(\bx_{k_i}^{(i)}\right). 
	\end{align}
	Thus, $ \hat{h}(F_c^*) $ can be rewritten as
	\begin{align}
		\hat{h}(F_c^*) = \sum\limits_{i} \frac{\lambda_i}{n_i} \sum\limits_{j \in \mJ_i} \phi \left(\bx_{j}^{(i)} \right)
	\end{align}
	From the definition of $ (f^c)^* $, we have 
	\begin{align}
		\sum\nolimits_{i} \lambda_i \phi \left(\bx^{(i)} \right) \leq  c \left(\bx^{(0)}, \ldots, \bx^{(N)} \right), 
	\end{align}
	Therefore, $ \phi $ is a feasible solution to Problem \ref{problem:newD}, and hence 
	\begin{align}
		\hat{h}(F_c^*) \leq \hat{h}(F_{\lambda}^*).
	\end{align}
	From $ \hat{h}(F_c^*) \geq \hat{h}(F_{\lambda}^*) $ and $ \hat{h}(F_c^*) \leq \hat{h}(F_{\lambda}^*) $, we have 
	\begin{align}
		\hat{h}(F_c^*) = \hat{h}(F_{\lambda}^*),
	\end{align}
	where $ \hat{h}(F^*_{\lambda}) = \{ \lambda_0 f^*, \ldots, \lambda_N f^* \} $.
\end{proof}

\newpage
\subsection{Proofs of Lemmas \ref{lemma:problem_II} and \ref{lemma:problem_III}}

\textbf{Lemma \ref{lemma:problem_II}} 
\emph{
	If the cost function $ c(\cdot, \ldots, \cdot) $ satisfies Definition \ref{def:cost_triangle_ineqn}, then $ \forall\; \bx_{k_i}^{(i)} \in \mX^{(i)}, i \in \mI $, if they are equal and $ f_i^*, i \in [N] $ are the optimizers to Problem \ref{problem:dist_D}, then $ (f^c)^*(\bx_{k_0}^{(0)}) = f_i^*(\bx_{k_i}^{(i)}), i \in [N] $, where $ (f^c)^* $ is the $ c $-conjugate function defined in Eqn. (\ref{eqn:c_conjugate}).
}

\begin{proof}
	We prove this by Contradiction.
	Without loss of generality, suppose $ \forall\; \bx_{k_i}^{(i)} \in \mX^{(i)}, i \in \mI $ are equal, and $ (f^c)^*(\bx_{k_0}^{(0)}) \neq f_i^*(\bx_{k_i}^{(i)}), i \in [N] $.
	Let $ \mK = \{ k_1, \ldots, k_N \} $ and $ \Omega = \{ \mX_1 \times \cdots \times \mX_N \} $.
	According to the definition of the $ c $-conjugate function, 
	\begin{align*}
		(f^c)^* \left( \bx_{k_0}^{(0)} \right) {=} \inf \left\{ - \sum\limits_{j=1}^{N} f_j^* \left(\bx^{(j)}_{k_j} \right), \inf_{\Omega \setminus \mK} \left\{ c \left(\bx_{k_0}^{(0)}, \bx^{(1)}, \ldots, \bx^{(N)} \right) - \sum\limits_{j=1}^{N} f_j^* \left(\bx^{(j)} \right) \right\} \right\}.
	\end{align*}
	For simplicity, let $ \psi^*(\bx^{(1)}, \ldots, \bx^{(N)}) = \sum\nolimits_{j=1}^{N} f_j^* \left(\bx^{(j)} \right) $, we rewrite the above function as
	\begin{align*}
		(f^c)^* \left( \bx_{k_0}^{(0)} \right) = \inf \left\{- \psi^*\left(\bx^{(1)}_{k_1}, \ldots, \bx^{(N)}_{k_N}\right), \inf_{\Omega \setminus \mK} \left\{ c \left(\bx_{k_0}^{(0)}, \bx^{(1)}, \ldots, \bx^{(N)} \right) - \psi^*\left(\bx^{(1)}, \ldots, \bx^{(N)}\right) \right\}\right\}.
	\end{align*}
	Since $ (f^c)^*(\bx_{k_0}^{(0)}) \neq f_i^*(\bx_{k_i}^{(i)}), i \in [N] $, for any $ \bx_i^{(0)} $, we have
	\begin{align}
		&\psi^*\left(\bx^{(1)}_{k_1}, \ldots, \bx^{(N)}_{k_N}\right) + c\left(\bx^{(0)}_i, \bx^{(1)}_{k_1}, \ldots, \bx^{(N)}_{k_N} \right) \nonumber \\
		> & \inf_{\Omega \setminus \mK} \left\{ c \left(\bx_{k_0}^{(0)}, \bx^{(1)}, \ldots, \bx^{(N)} \right) - \psi^*(\bx^{(1)}, \ldots, \bx^{(N)}) \right\} + c\left(\bx^{(0)}_i, \bx^{(1)}_{k_1}, \ldots, \bx^{(N)}_{k_N} \right) \nonumber \\
		=& \inf_{\Omega \setminus \mK} \left\{ - \psi^*\left(\bx^{(1)}, \ldots, \bx^{(N)}\right) + c \left(\bx_{k_0}^{(0)}, \bx^{(1)}, \ldots, \bx^{(N)} \right)  + c\left(\bx^{(0)}_i, \bx^{(1)}_{k_1}, \ldots, \bx^{(N)}_{k_N} \right) \right\} \nonumber \\
		\geq& \inf_{\Omega \setminus \mK} \left\{ - \psi^*\left(\bx^{(1)}, \ldots, \bx^{(N)}\right) + c\left(\bx^{(0)}_i, \bx^{(1)}, \ldots, \bx^{(N)} \right) \right\}. \label{ineqn:lemma3} 
	\end{align}
	Line \ref{ineqn:lemma3} follows the fact that the definition of the cost function, $ \bx_{k_i}^{(i)} \in \mX^{(i)}, i \in \mI $ are equal.
	Suppose the number of samples in each distribution is $ n $,
	\begin{align*}
		& W^*(\mu_0, \ldots, \mu_N) = \sup_{F_c} \; \hat{h} (F_c) \\
		=& \frac{1}{n} \sum\limits_{j \in \mJ_0} (f^c)^*\left(\bx_{j}^{(0)} \right) + \sum\limits_{i=1}^N \frac{1}{n} \sum\limits_{j \in \mJ_i} f_i^*\left(\bx_{j}^{(i)} \right) \\
		=& \frac{1}{n} \sum\limits_{j \in \mJ_0} \inf_{\Omega \setminus \mK} \left\{ c \left(\bx_{j}^{(0)}, \bx^{(1)}, \ldots, \bx^{(N)} \right) - \psi^*\left(\bx^{(1)}, \ldots, \bx^{(N)}\right) \right\} + \sum\limits_{i=1}^N \frac{1}{n} \sum\limits_{j \in \mJ_i} f_i^*\left(\bx_{j}^{(i)} \right) \\
		=& \frac{1}{n} \sum\limits_{j \in \mJ_0} \inf_{\Omega \setminus \mK} \left\{ c \left(\bx_{j}^{(0)}, \bx^{(1)}, \ldots, \bx^{(N)} \right) - \psi^*\left(\bx^{(1)}, \ldots, \bx^{(N)}\right) \right\} \\
		&+ \frac{1}{n} \psi^* \left( \bx^{(1)}_{k_1}, \ldots, \bx^{(N)}_{k_N} \right) + \frac{1}{n}  \sum\limits_{j: \bx_j^{(i)} \notin \mK, i\in[N]} \psi^*\left(\bx^{(1)}_j, \ldots, \bx^{(N)}_j \right).
	\end{align*}
	We can always find another function $ \psi' $, such that $ \psi'\left(\bx^{(1)}, \ldots, \bx^{(N)}\right) = \psi^*\left(\bx^{(1)}, \ldots, \bx^{(N)}\right) $ for $ \forall \bx^{(1)}, \ldots, \bx^{(N)} \in \Omega \setminus \mK $, and 
	\begin{align*}
		{-}\psi^*\left(\bx^{(1)}_{k_1}, \ldots, \bx^{(N)}_{k_N}\right) {>} {-} \psi'\left(\bx^{(1)}_{k_1}, \ldots, \bx^{(N)}_{k_N}\right) {>}  \inf_{\Omega \setminus \mK} \left\{ c \left(\bx_{k_0}^{(0)}, \bx^{(1)}, \ldots, \bx^{(N)} \right) {-} \psi^*\left(\bx^{(1)}, \ldots, \bx^{(N)}\right) \right\}.
	\end{align*}
	In this case, $ (f^c)'\left(\bx_{j}^{(0)}\right) = (f^c)^*\left(\bx_{j}^{(0)}\right), \forall j \in \mJ_0 $, but
	\begin{align*}
		\psi^*\left(\bx^{(1)}_{k_1}, \ldots, \bx^{(N)}_{k_N}\right) < \psi'\left(\bx^{(1)}_{k_1}, \ldots, \bx^{(N)}_{k_N}\right).
	\end{align*}
	Therefore, $ \hat{F_{\psi}} > \hat{F_c} $, a contradiction.
\end{proof}

\newpage
\textbf{Corollary \ref{lemma:problem_II_pro}}
\emph{
	Given the cost function $ c(\cdot, \ldots, \cdot) $ defined in Definition \ref{def:cost_triangle_ineqn}, we assume $ f^* $ is $ 1 $-Lipschitz continuous, the samples are bounded and the distance function satisfies $ d(\bx, \bz) {\leq} d(\bx, \by) {+} d(\by, \bz) {\leq} C d(\bx, \bz) $, where $ C{\ge} 1 $, when $ N{+}1 $ samples $ \bx_{k_i}^{(i)} {\in} \mX^{(i)}, i {\in} \mI $ are close to each other, \ie $ c(\bx_{k_0}^{(0)}, \ldots, \bx_{k_N}^{(N)}) $ is arbitrarily small, then $ (f^c)^*(\bx_{k_0}^{(0)}) $ would be arbitrarily close to $ f_i^*(\bx_{k_i}^{(i)}), i {\in} [N] $, where $ f_i^*, i {\in} [N] $ are the optimizers to Problem \ref{problem:dist_D}, and $ (f^c)^* $ is the $ c $-conjugate function defined in Eqn. (\ref{eqn:c_conjugate}).
}

\begin{proof}
	We prove the case of $ N{=}1 $, it can be directly extended to the case of $ N{>}1 $.
	Specifically, the potential function $ f_1 {:=} {-} f $. 
	Based on the definition of the cost function, we have $ c(\bx, \by) {:=} d(\bx, \by) $.
	If $ \bx_{k_1}^{(1)} $ is the optimal solution, \ie $ (f^c)^* \left(\bx_{k_0}^{(0)}\right) {=} f^*\left(\bx_{k_1}^{(1)}\right) {+} d \left(\bx_{k_0}^{(0)}, \bx_{k_1}^{(1)} \right) $, then
	\begin{align}
		\left| (f^c)^*\left(\bx_{k_0}^{(0)}\right) - f^*\left(\bx_{k_1}^{(1)}\right) \right| = d \left(\bx_{k_0}^{(0)}, \bx_{k_1}^{(1)} \right).
	\end{align}
	If $ \bx_{k'_1}^{(1)} $ is the optimal solution, \ie $ (f^c)^* \left(\bx_{k_0}^{(0)}\right) {=} f^*\left(\bx_{k'_1}^{(1)}\right) {+} d \left(\bx_{k_0}^{(0)}, \bx_{k'_1}^{(1)} \right) $, then
	\begin{align}
		\left| (f^c)^*\left(\bx_{k_0}^{(0)}\right) - f^*\left(\bx_{k_1}^{(1)}\right) \right| 
		=& \left| f^*\left(\bx_{k'_1}^{(1)}\right) + d \left(\bx_{k_0}^{(0)}, \bx_{k'_1}^{(1)} \right) - f^*\left(\bx_{k_1}^{(1)}\right) \right| \\
		\leq& \left| f^*\left(\bx_{k'_1}^{(1)}\right) - f^*\left(\bx_{k_1}^{(1)}\right) \right| + d \left(\bx_{k_0}^{(0)}, \bx_{k'_1}^{(1)} \right) \\
		\leq& d \left(\bx_{k'_1}^{(1)}, \bx_{k_1}^{(1)}\right) + d \left(\bx_{k_0}^{(0)}, \bx_{k'_1}^{(1)} \right) \\
		\leq& C d \left(\bx_{k_0}^{(0)}, \bx_{k_1}^{(1)} \right).
	\end{align}
	For the above two cases, when $ \bx_{k_1}^{(1)} $ is close to $ \bx_{k_0}^{(0)} $, then $ f^*\left(\bx_{k_1}^{(1)}\right) $ is also close to $ (f^c)^*\left(\bx_{k_0}^{(0)}\right) $.
\end{proof}

\textbf{Lemma \ref{lemma:problem_III}} 
\emph{
	Suppose $ f^* $ is an optimal solution to Problem \ref{problem:newD} and $ \sum_{i\in\mI} \lambda_i = 0 $, then $ f^* $ satisfies 
	\begin{align}
		f^*\left(\bx^{(0)}\right) = \inf_{\bx^{(1)}, \ldots, \bx^{(N)}} \left\{c \left(\bx^{(0)}, \bx^{(1)}, \ldots, \bx^{(N)} \right) - \sum\nolimits_{j=1}^{N} \lambda_j f \left(\bx^{(j)} \right) \right\}, \quad \forall \; \bx \in \mX^0.
	\end{align}
}
\begin{proof}
	Since $ f^* $ is the optimal solution to Problem \ref{problem:newD}, we have 
	\begin{align}
		f^*\left(\bx^{(0)}\right) \leq \inf_{\bx^{(1)}, \ldots, \bx^{(N)}} \left\{c \left(\bx^{(0)}, \bx^{(1)}, \ldots, \bx^{(N)} \right) - \sum\nolimits_{j=1}^{N} \lambda_j f \left(\bx^{(j)} \right) \right\}, \quad \forall \; \bx \in \mX^0.
	\end{align}
	We prove by contradiction.
	Without loss of generality, suppose there exists a $ \bx_{k_0}^{(0)} $, such that 
	\begin{align} \label{ineqn:c_conjugate_contradiction}
		f^*\left(\bx^{(0)}\right) < \inf_{\bx^{(1)}, \ldots, \bx^{(N)}} \left\{c \left(\bx^{(0)}, \bx^{(1)}, \ldots, \bx^{(N)} \right) - \sum\nolimits_{j=1}^{N} \lambda_j f \left(\bx^{(j)} \right) \right\}, \quad \forall \; \bx \in \mX^0.
	\end{align}
	Note that $ \bx_{k_i}^{(i)} \in \mX^{(i)}, i \in \mI $ can not be equal, otherwise,
	\begin{align}
		f^* \left(\bx_{k_0}^{(0)}\right) 
		=& -\sum\nolimits_{i\in[N]} \lambda_i f^* \left(\bx_{k_i}^{(i)}\right) \\
		=& -\sum\nolimits_{i\in[N]} \lambda_i f^* \left(\bx_{k_i}^{(i)}\right) + c\left(\bx_{k_0}^{(0)}, \bx_{k_1}^{(1)}, \ldots, \bx_{k_N}^{(N)}\right)  \\
		\geq& \inf_{\bx^{(1)}, \ldots, \bx^{(N)}} \left\{ -\sum\nolimits_{i\in[N]} \lambda_i f^* \left(\bx^{(i)}\right) + c\left(\bx^{(0)}, \bx^{(1)}, \ldots, \bx^{(N)}\right) \right\}.
	\end{align}
	It is not consistent with Eqn. (\ref{ineqn:c_conjugate_contradiction}), thus $ \bx_{k_i}^{(i)} \in \mX^{(i)}, i \in \mI $ can not be equal.
	
	Therefore, there exists another function $ f' $ such that $ f'\left( \bx_{j}^{(i)} \right) = f^* \left( \bx_{j}^{(i)} \right), \forall\; \bx_{j}^{(i)} \in \mX^i, i\in[N] $, and $ f'\left( \bx_{j}^{(0)} \right) = f^* \left( \bx_{j}^{(0)} \right), \forall\; \bx_{j}^{(0)} \in \mX^0 \setminus \bx_{k_0}^{(0)} $ and 
	\[f'\left( \bx_{k_0}^{(0)} \right) = \inf_{\bx^{(1)}, \ldots, \bx^{(N)}} \left\{c \left(\bx^{(0)}, \bx^{(1)}, \ldots, \bx^{(N)} \right) - \sum\nolimits_{j=1}^{N} \lambda_j f \left(\bx^{(j)} \right) \right\}.\]
	It is easy to verify that $ f' $ satisfies the constraints in Problem \ref{problem:newD} and $ \hat{h}(F'_{\lambda}) {>} \hat{h}(F^*_{\lambda}) $, where $ \hat{h}(F'_{\lambda}) {=} \{ \lambda_0 f', \ldots, \lambda_N f' \} $ and $ \hat{h}(F^*_{\lambda}) {=} \{ \lambda_0 f^*, \ldots, \lambda_N f^* \} $.
	Therefore, it leads to a contradiction.
\end{proof}

\newpage
\section{Proof of Theorem \ref{thm:solution_exist}} \label{sec:solution_exist}

\textbf{Theorem \ref{thm:solution_exist} }
\emph{
	Suppose the domains are connected, $ c $ is continuously differentiable and that each $ \mu_i $ is absolutely continuous. 
	If $ (f_0, \ldots, f_N) $ and $ (\lambda_0f, \ldots, \lambda_{N}f ) $ are solutions to Problems \ref{problem:MK-D}, then there exist some constant $ \varepsilon_i $ for all $ i \in \mI $ such that $ \sum_{i} \varepsilon_i = 0 $, and $ f_i = \lambda_i f + \varepsilon_i $.
}

\begin{proof}
	First, using a convexification trick \cite{ref_gangbo1998optimal}, we are able to construct a $ c $-conjugate solution $ (f_0^c, f_1^c, \ldots, f_N^c) $ to the continuous case of Problem \ref{problem:MK-D}, where $ f_i^c $ can be defined as:
	\begin{align}\label{eqn:conjugate_ui}
		f_i^c \left( X^{(i)} \right) = \inf \left\{ c \left( X^{(0)}, \ldots, X^{(N)} \right) - \sum_{0 \leq j < i} f_j^c\left(X^{(j)}\right) - \sum_{i < j \leq N} f_j\left(X^{(j)}\right) \right\}, \quad \forall i \in \mI.
	\end{align}
	Based on the definition of $ f_i^c $ and its optimality, we have
	\begin{align}\label{ieqn:conjugate_ui1}
		f_i^c \left( X^{(i)} \right) \leq \inf \left\{ c \left( X^{(0)}, \ldots, X^{(N)} \right) - \sum_{j \neq i} f_j^c \left( X^{(j)} \right)  \right\}, \quad \forall i \in \mI.
	\end{align}
	Then we iteratively obtain that $ f_i \left( X^{i} \right) \leq f_i^c \left( X^{i} \right) $, and using the definition of Equation (\ref{eqn:conjugate_ui}),
	\begin{align}\label{ieqn:conjugate_ui2}
		f_i^c \left( X^{(i)} \right) =& \inf \left\{ c \left( X^{(0)}, \ldots, X^{(N)} \right) - \sum_{0 \leq j < i} f_j^c\left(X^{(j)}\right) - \sum_{i < j \leq N} f_j\left(X^{(j)}\right) \right\}, \quad \forall i \in \mI \nonumber \\
		\ge& \inf \left\{ c \left( X^{(0)}, \ldots, X^{(N)} \right) - \sum_{j \neq i} f_j^c\left(X^{(j)}\right) \right\}.
	\end{align}
	Combining Inequalities (\ref{ieqn:conjugate_ui1}) and (\ref{ieqn:conjugate_ui2}), we have a $ c $-conjugate solution $ \left( f_0^c, f_1^c, \dots, f_N^c \right) $ which satisfies Definition \ref{def:conjugate}.
	Let $ \varphi_i = \lambda_i f, -1 {\leq} \lambda {\leq} 1 $, using the convexification trick, we are able to find $c$-conjugate solutions $ (f_0^c, \ldots, f_N^c) $  and $ (\varphi_0^c, \ldots, \varphi_N^c) $ to the continuous cases of Problems \ref{problem:dist_D} and \ref{problem:newD} such that $ f_i \leq f_i^c $ and $ \varphi_i \leq \varphi_i^c $. As
	\begin{align*}
		\sum_{i \in \mI} \int  f_i \left(X^{(i)}\right) d \mu_i \left(X^{(i)}\right) = \sum_{i \in \mI} \int  f_i^c \left(X^{(i)}\right) d \mu_i \left(X^{(i)}\right),
	\end{align*}
	we must have $ f_i = f_i^c $, $ \mu_i $ almost everywhere. Similarly, $ \varphi_i = \varphi_i^c $, $ \mu_i $ almost everywhere.
	We choose $ X^{(i)} \in \mD_i $ where $ f_i^c $ and $ \varphi_i^c $ are differentiable.
	Then there exist $ X^{(j)} $ for all $ j \neq i $ such that $ ( X^{(0)}, \dots, X^{(i-1)}, X^{(i)}, X^{(i+1)}, \ldots, X^{(N)} ) $ in the support of $ \mu $.
	According to Theorem \ref{thm:pd_opt}, we have 
	\begin{align*}
		f_i^c \left( X^{(i)} \right) - c \left( X^{(0)}, \dots, X^{(i-1)}, X^{(i)}, X^{(i+1)}, \ldots, X^{(N)} \right) = - \sum_{j \neq i} f_j^c \left(X^{(j)} \right).
	\end{align*}
	Because $ f_i^c \left( Z^{(i)} \right) {-} c \left( X^{(0)}, \dots, X^{(i{-}1)}, Z^{(i)}, X^{(i)}, \ldots, X^{(N)} \right) \leq {-} \sum_{j \neq i} f_j^c \left(X^{(j)} \right)$ for all other $ Z^{(i)} $ we have the differential of $ f_i^c $ and $ c(\cdot) $ \wrt $ X^{(i)} $ as follows
	\begin{align*}
		D_{X^{(i)}} f_i^c \left( X^{(i)} \right) = D_{X^{(i)}} c \left( X^{(0)}, \dots, X^{(i-1)}, X^{(i)}, X^{(i+1)}, \ldots, X^{(N)} \right).
	\end{align*}
	Similarly, we have
	\begin{align*}
		D_{X^{(i)}} \varphi_i^c \left( X^{(i)} \right) = D_{X^{(i)}} c \left( X^{(0)}, \dots, X^{(i-1)}, X^{(i)}, X^{(i+1)}, \ldots, X^{(N)} \right).
	\end{align*}
	Therefore, we have $ D_{X^{(i)}} f_i^c \left( X^{(i)} \right) = D_{X^{(i)}} \varphi_i^c \left( X^{(i)} \right) $.
	As this equality holds for almost all $ X^{(i)} $, we have $ f_i^c \left( X^{(i)} \right) = \varphi_i^c \left( X^{(i)} \right) + \varepsilon_i $ and $ f_i \left( X^{(i)} \right) = \varphi_i \left( X^{(i)} \right) + \varepsilon_i $.
	Choosing any $ \left( X^{(0)}, \dots, X^{(i-1)}, X^{(i)}, X^{(i+1)}, \ldots, X^{(N)} \right) $ in the support of $ \gamma $, then
	\begin{align*}
		\sum_{i \in \mI} f_i^c \left( X^{(i)} \right) = c \left( X^{(0)}, \dots, X^{(i-1)}, X^{(i)}, X^{(i+1)}, \ldots, X^{(N)} \right) = \sum_{i \in \mI} \varphi_i^c \left( X^{(i)} \right).
	\end{align*}
	Therefore, $ \sum_i \varepsilon_i = 0 $.
\end{proof}

\section{Error Bound of New Dual Formulation} \label{sec:error_bound}
\begin{thm}\textbf{\emph{(Error bound)}}
	Suppose the function $ f $ is an optimal solution to Problem \ref{problem:newD} and is bounded in $ [{-}\Delta, \Delta] $, we let $ \hat{\sigma}_{k_0, \ldots, k_N} {=} \1_{[f \notin \Omega]} $, where $ \Omega {=} \{ f | {\sum_{i} \lambda_i f(\bx^{(i)}_{k_i} ) \leq  c (\bx^{(0)}_{k_0}, \ldots, \bx^{(N)}_{k_N} )} \} $ 
	with $-1 {\leq} \lambda_i {\leq} 1$, and $ \sigma {=} \mmE[\hat{\sigma}_{k_0, \ldots, k_N}] $ be the expectation of the probability that violates constraints in Problem \ref{problem:newD}. 
	Define $ h(F_{\lambda}) {=} \sum_i \lambda_i \mmE_{\bx^{(i)}} [f(\bx^{(i)})] $, then the error bound between the discrete $ \hat{h} $ and the continuous problem $ h $ is:
	\begin{small}
		\begin{align}
			P \left( \left| \hat{h}(F_{\lambda}) {-} h(F_{\lambda}) \right| {\leq} \epsilon \right) {>} 1 {-} 2(N{+}1)\exp \left( \frac{-\underline{n} \epsilon^2}{2(N{+}1)^2 \Delta^2} \right),
		\end{align} 
	\end{small}%
	where $ \underline{n} {=} \min_{i {\in} \mI} n_i $. 
	Let $ M {=} \prod\nolimits_{i {\in} \mI} {n_i} $, then we have $ P(|\sigma| {>} \epsilon) {\leq} 2e^{-2M\epsilon^2} $. 
\end{thm}

\begin{proof}
	Based on the definitions of $ \hat{h}(F_{\lambda}) $ and $ {h}(F_{\lambda}) $ and $ -1 {\leq} \lambda_i {\leq} 1, i {\in} \mI $, 
	\begin{small}
		\begin{align} \label{ieqn:hat_h-h}
			\left| \hat{h}\left(F_\lambda\right) - h\left(F_\lambda\right)  \right| =& \left| \sum_i \frac{\lambda_i}{n_i} \sum_{j \in \mJ_i} f\left( \bx_j^{(i)} \right) - \sum_i \lambda_i \mmE \left[ f\left( \bx^{(i)} \right) \right] \right| \nonumber \\
			=& \left| \sum_i \lambda_i \left( \frac{1}{n_i} \sum_{j \in \mJ_i} f\left( \bx_j^{(i)} \right) - \mmE \left[ f\left( \bx^{(i)} \right) \right] \right) \right|, -1 \leq \lambda_i \leq 1 \nonumber \\
			\leq&  \sum_i \left| \frac{1}{n_i} \sum_{j \in \mJ_i} f\left( \bx_j^{(i)} \right) - \mmE \left[ f\left( \bx^{(i)} \right) \right] \right|. 
		\end{align}%
	\end{small}
	Suppose the function $ f $ is bounded in $ [-\Delta, \Delta] $, then, according to Hoeffding's inequality, we have
	\begin{small}
		\begin{align}\label{ieqn:hoeffding}
			P \left( \left| \frac{1}{n_i} \sum_j f\left( \bx_j^{(i)} \right) - \mmE \left[ f\left( \bx^{(i)} \right) \right] \right| > \frac{\epsilon}{N{+}1} \right) \leq 2\exp \left( \frac{-n_i \epsilon^2}{2(N{+}1)^2 \Delta^2} \right).
		\end{align}
	\end{small}
	Using Inequality (\ref{ieqn:hoeffding}) and union bound over all $ i \in \mI $, we further have the following inequality, 
	\begin{small}
		\begin{align*}
			&P \left( \bigcup_{i \in \mI} \left( \left| \frac{1}{n_i} \sum_j f\left( \bx_j^{(i)} \right) - \mmE \left[ f\left( \bx^{(i)} \right) \right] \right| > \frac{\epsilon}{N+1} \right) \right) \\
			\leq& \sum_i P \left( \left| \frac{1}{n_i} \sum_j f\left( \bx_j^{(i)} \right) - \mmE \left[ f\left( \bx^{(i)} \right) \right] \right| > \frac{\epsilon}{N+1} \right) \\
			\leq& 2 (N+1) \exp \left( \frac{-\underline{n} \epsilon^2}{2 (N+1)^2 \Delta^2} \right),
		\end{align*}
	\end{small}
	where $ \underline{n} = \min_{i \in \mI} n_i $.
	Equivalently, we rewrite the above inequality as
	\begin{small}
		\begin{align*}
			P \left( \bigcap_{i\in \mI} \left( \left| \frac{1}{n_i} \sum_j f\left( \bx_j^{(i)} \right) - \mmE \left[ f\left( \bx^{(i)} \right) \right] \right| \leq \frac{\epsilon}{N+1} \right) \right) > 1 - 2 (N+1) \exp \left( \frac{-\underline{n} \epsilon^2}{2(N+1)^2 \Delta^2} \right).
		\end{align*}
	\end{small}
	Therefore, from Inequality (\ref{ieqn:hat_h-h}), the following probability inequality satisfies:
	\begin{small}
		\begin{align*}
			P \left( \left| \hat{h}\left({F_{\lambda}}\right) - h\left({F_{\lambda}}\right)  \right| \leq \epsilon \right) &\ge P \left( \sum_i \left| \frac{1}{n_i} \sum_j f\left( \bx_j^{(i)} \right) - \mmE \left[ f\left( \bx^{(i)} \right) \right] \right| \leq \epsilon \right) \\
			&> 1 - 2(N+1) \exp \left( \frac{-\underline{n} \epsilon^2}{2(N+1)^2 \Delta^2} \right).
		\end{align*}
	\end{small}%
	Based on the definitions of $ \hat{\sigma}_{k_0, \ldots, k_N} $ and $ \sigma $, they are bounded in the interval $ [0, 1] $. 
	Using Hoeffding's inequality, we have
	\begin{small}
		\begin{align*}
			P \left( \left| \frac{1}{M} \sum_{k_0, \ldots, k_N} \hat{\sigma}_{k_0, \ldots, k_N} - \sigma \right| > \epsilon \right) \leq 2 \exp\left(-2M\epsilon^2\right),
		\end{align*}
	\end{small}%
	where $ M = \prod\nolimits_{i} {n_i} $.
	Suppose the function $ f $ can be learned by a deep neural network with sufficient capacity, and it is able to solve Problem \ref{problem:newD}. 
	Then, the inequality constraints $ \sum_{i} f\left(\bx^{(i)}_{k_i} \right) \leq c \left(\bx^{(0)}_{k_0}, \ldots, \bx^{(N)}_{k_N} \right) $ are satisfied, and thus $ \hat{\sigma}_{k_0, \ldots, k_N} = 0 $ for $ \forall\; {k_0, \ldots, k_N} $, we have 
	\[P(|\sigma| > \epsilon) \leq 2e^{-2M\epsilon^2}.\]
\end{proof}

\section{Proof of Theorem \ref{thm:mwgan_opt}} \label{sec:mwgan_opt}

\begin{deftn} \textbf{(Function space)}
	Let $ \mX \subseteq \mmR^d $ be a compact set (such as $ [0, 1]^d $ the space of images), the function space can be defined as
	\begin{align} \label{def:C_b-f}
		C_b(\mX) = \{ f: \mX \to \mmR, f \text{ is continuous and bounded} \}.
	\end{align}
\end{deftn}

\begin{ass} \cite{ref_arjovsky2017wasserstein} \label{ass:wgan_ass1}
	Let $ g: \mX \to \mX $ be locally Lipschitz between finite dimensional vector spaces. 
	Given $ g_{\theta} (X) $ evaluated on coordinates $ (X, \theta) $, 
	we say that $ g $ satisfies assumption \ref{ass:wgan_ass1} for a certain probability distribution $ p $ over $ \mX $ if there are local Lipschitz constants $ L(\theta, \bx) $ such that
	\begin{align*}
		\mmE_{x \sim \mmP} [L(\theta, \bx)] < +\infty.
	\end{align*}
\end{ass}
\begin{ass} \cite{ref_arjovsky2017wasserstein}
	Assume the discriminator $ f $ is Lipschitz continuous \wrt $ \bx $.
\end{ass}
\begin{lemma}\cite{ref_arjovsky2017wasserstein}\label{lamma:gradient_trans}
	Assume the discriminator $ f $ is $ 1 $-Lipschitz \wrt $ w $ and the generator $ g_{\theta} (x) $ is locally Lipschitz as a function of $ (\theta, x) $, then $ \nabla_{\theta} \mmE_{\bx \sim \mmP_s} [ f(g_{\theta}(\bx)) ] = \mmE_{\bx \sim \mmP_s} \left[ \nabla_{\theta} f(g_{\theta} (\bx))\right] $.
\end{lemma}
\textbf{Theorem \ref{thm:mwgan_opt}} 
\emph{
	If each generator $ g_{i} {\in} \mG, i {\in} [N] $ is locally Lipschitz and satisfies Assumption 1 \cite{ref_arjovsky2017wasserstein}
	, then there exists a discriminator $ f $ to Problem \ref{problem:MWGAN}, we have the gradient $ \nabla_{\theta_i} W(\hat{\mmP}_s, \hat{\mmP}_{\theta_1}, \ldots, \hat{\mmP}_{\theta_N}) = {-} \lambda_i^{+} \mmE_{{\bx}\sim \hat{\mmP}_s} \left[ \nabla_{\theta_i} f(g_{i}({\bx})) \right] $ for all $ \theta_i, i\in[N] $ when all terms are well-defined.  
}

\begin{proof}
	Recall the optimization problem, we first define the value function as follows:
	\begin{align*}
		V(\tilde{f}, \theta) 
		&= {\mmE}_{\bx \sim \mmP_{s}} \left[ \tilde{f} (\bx) \right] - \frac{1}{N} \sum_i {\mmE}_{\bx \sim \mmP_{\theta_i}} \left[ \tilde{f}(\bx) \right] \\
		&= {\mmE}_{\bx \sim \mmP_{s}} \left[ \tilde{f} (\bx) \right] - \frac{1}{N} \sum_i {\mmE}_{\bx \sim \mmP_{s}} \left[ \tilde{f}({g_i} (\bx)) \right],
	\end{align*}
	where $ \theta $ is the set of $ \theta_i, i\in[N] $, $ \tilde{f} $ lies in $ \tilde{\mF}=\{ \tilde{f}: \mX \to \mmR, \tilde{f} \in C_b(\mX), \tilde{f} \in \Omega \} $ and $ C_b(\mX) $ is defined in (\ref{def:C_b-f}). 
	\begin{align*}
		\Omega = \left\{ \tilde{f} : \tilde{f}(\bx) - \frac{1}{N}\sum\nolimits_{i \in [N]} \tilde{f} \left(\bx^{(i)}\right) \leq c\left(\bx^{(0)}, \bx^{(1)}, \ldots, \bx^{(N)}\right) \right\}, 
	\end{align*}
	where $ \bx^{(0)} {:=} \bx $.
	Since $ \mX $ is compact, and based on Theorem \ref{thm:pd_opt}, there is a solution $ f {\in} \tilde{\mF} $ that satisfies
	\begin{align*}
		W({\mmP}_s, {\mmP}_{\theta_1}, {\ldots}, {\mmP}_{\theta_N}) = \sup_{\tilde{f} \in \tilde{\mF}} V(\tilde{f}, \theta) = V(f, \theta).
	\end{align*}
	Define the optimal set $ \mF^*(\theta) = \{ f \in \tilde{\mF}: V(f, \theta) = W({\mmP}_s, {\mmP}_{\theta_1}, {\ldots}, {\mmP}_{\theta_N}) \} $, and note that this set $ \mF^* (\theta) $ is non-empty.
	Based on envelope theorem \cite{ref_milgrom2002envelope}, we have
	\begin{align*}
		\nabla_{\theta_i} W({\mmP}_s, {\mmP}_{\theta_1}, {\ldots}, {\mmP}_{\theta_N}) = \nabla_{\theta_i} V(f, \theta)
	\end{align*}
	for any $ f \in \mF^*(\theta) $ when all terms are well-defined.
	Note that $ f $ exists since $ \mF^*(\theta) $ is non-empty for all $ \theta_i $. Then, we have
	\begin{align*}
		\nabla_{\theta_i} W({\mmP}_s, {\mmP}_{\theta_1}, {\ldots}, {\mmP}_{\theta_N})
		=& \nabla_{\theta_i} V(f, \theta) \\
		=& \nabla_{\theta_i} \left[ \mmE_{\bx \sim \mmP_{s}} \left[ {f} (\bx) \right] - \frac{1}{N} \sum_i \mmE_{\bx \sim \mmP_{s}} \left[ {f}({g_i} (\bx)) \right] \right] \\
		=& \nabla_{\theta_i} \left[ \mmE_{\bx \sim \mmP_{s}} \left[ {f} (\bx) \right] - \frac{1}{N} \mmE_{\bx \sim \mmP_{s}} \left[ {f}({g_i}_{} (\bx)) \right] \right] \\
		=& - \frac{1}{N} \nabla_{\theta_i} \mmE_{\bx \sim \mmP_{s}} \left[ {f}({g_i}_{} (\bx)) \right] \\
		=& - \frac{1}{N} \mmE_{\bx \sim \mmP_{s}} \left[ \nabla_{\theta_i} {f}({g_i}_{} (\bx)) \right],
	\end{align*}
	where the last equality holds by Lemma \ref{lamma:gradient_trans}.
\end{proof}

\section{Proof of Theorem \ref{thm:generalization}} \label{sec:generalization}
\textbf{Theorem \ref{thm:generalization} \textbf{(Generalization bound) }} 
\emph{
	Let $ \mmP_s $ and $ \mmP_{\theta_i} $ be the continuous real and generated distributions, and $ \hat{\mmP}_s $ and $ \hat{\mmP}_{\theta_i} $ be the empirical real and generated distributions with at least $ n $ samples each. When $ n \ge \frac{C \kappa \Delta^2 \log(L\kappa/\epsilon) }{\epsilon^2} $, the following generalization bound is satisfied with probability at least $ 1-e^{-\kappa} $,
	\begin{align*}
		\left|W\left(\hat{\mmP}_s, \hat{\mmP}_{\theta_1}, {\ldots}, \hat{\mmP}_{\theta_N}\right) - W(\mmP_s, \mmP_{\theta_1}, {\ldots}, \mmP_{\theta_N}) \right| \leq \epsilon.
	\end{align*}
}

\begin{proof}
	Let $ \tilde{\mW} $ be a finite set such that every point $ w \in \mW $ is within distance $ \frac{\epsilon}{8L} $ of a point $ w' \in \tilde{\mW} $, \ie for every $ w \in \mW $, there exist a $ w' \in \mW $ such that $ \| w - w' \| \leq \frac{\epsilon}{8L} $. For any $ \bx \in \mmP_s $ or $ \bx \in \hat{\mmP}_s $, assume that $ f $ is $ L $-Lipschitz continuous \wrt $ w $, then we have
	\begin{align} \label{ieqn:f_w_Lip}
		|f_{w'}(\bx) - f_{w}(\bx)| \leq L \| w' - w \| \leq \frac{\epsilon}{8}.
	\end{align}
	Assume that $ f $ is bounded in $ [-\Delta, \Delta] $. 
	Using to Hoeffding's inequality, for every $ w' \in \tilde{\mW} $, we have
	\begin{align*}
		P \left( \left| \mmE_{\bx \sim \mmP_{s}} \left[ f_{w'}(\bx) \right] - \mmE_{\bx \sim \hat{\mmP}_{s}} \left[ f_{w'}(\bx) \right] \right| \ge \frac{\epsilon}{4} \right) \leq 2 \exp\left( -\frac{n\epsilon^2}{32 \Delta^2} \right).
	\end{align*}
	Therefore, when $ n \ge \frac{C\kappa \Delta^2 \log(L\kappa/\epsilon) }{\epsilon^2} $ for a large enough constant $ C $, we have union bounds over all $ w' \in \tilde{\mW} $. 
	Then, we have $ |\mmE_{\bx \sim \mmP_{s}} f_{w'}(\bx) - \mmE_{\bx \sim \hat{\mmP}_{s} } f_{w'}(\bx)| \leq \frac{\epsilon}{4} $ with the high probability at least $ 1-\exp(-\kappa) $, where $ \kappa $ is the number of parameters in the discriminator $ f $.
	For every $ w \in \mW $, we can find a $ w' \in \tilde{\mW} $ such that the following satisfies
	\begin{align*}
		\left| \mmE_{\bx {\sim} \mmP_{s}} \left[ f_w(\bx) \right] {-} \mmE_{\bx {\sim} \hat{\mmP}_{s}} \left[ f_w(\bx) \right] \right| 
		\leq& \left| \mmE_{\bx {\sim} \mmP_{s}} \left[ f_{w'}(\bx) \right] {-} \mmE_{\bx {\sim} {\mmP}_{s}} \left[ f_w(\bx) \right] \right| 
		{+} \left| \mmE_{\bx {\sim} \hat{\mmP}_{s}} \left[ f_{w'}(\bx) \right] {-} \mmE_{\bx {\sim} \hat{\mmP}_{s}} \left[ f_w(\bx) \right] \right| \\
		&{+} \left| \mmE_{\bx {\sim} \mmP_{s}} \left[ f_{w'}(\bx) \right] {-} \mmE_{\bx {\sim} \hat{\mmP}_{s}} \left[ f_{w'}(\bx) \right] \right| \\
		\leq& \frac{\epsilon}{8} {+} \frac{\epsilon}{8} {+} \frac{\epsilon}{4} 
		\leq \frac{\epsilon}{2}.
	\end{align*}
	The third line holds by Inequality (\ref{ieqn:f_w_Lip}). 
	Therefore, with high probability at least $ 1-\exp(-\kappa) $, for every discriminator $ f_w $, 
	\begin{align*}
		\left| \mmE_{\bx \sim \mmP_{s}} \left[ f_w(\bx) \right] - \mmE_{\bx \sim \hat{\mmP}_{s}} \left[ f_w(\bx) \right] \right| \leq \frac{\epsilon}{2}. 
	\end{align*}
	Similarly, for $ \mmP_{\theta_i} $ and $ \hat{\mmP}_{\theta_i} $, when $ n {\ge} \frac{C \kappa \Delta^2 \log(L\kappa/\epsilon) }{\epsilon^2} $, with the probability at least $ 1{-}\exp({-}\kappa) $, 
	\begin{align*}
		\left| \mmE_{\bx \sim \mmP_{\theta_i}} \left[ f_w(\bx) \right] - \mmE_{\bx \sim \hat{\mmP}_{\theta_i}} \left[ f_w(\bx) \right] \right| \leq \frac{\epsilon}{2}, \quad \forall i=1, \ldots, N.
	\end{align*}
	Let $ f_w $ be the optimal discriminator of $ W(\mmP_s{,} \mmP_{\theta_1}{,} {\ldots}, \mmP_{\theta_N}) $, we have
	\begin{align*}
		W \left(\hat{\mmP}_s, \hat{\mmP}_{\theta_1}, \ldots, \hat{\mmP}_{\theta_N}\right) =& \sup_{f\in \mF } \mathop{{\mmE}}\nolimits_{\bx \sim \hat{\mmP}_s} \left[ f (\bx) \right] - \frac{1}{N} \sum_i \left[ \mathop{{\mmE}}\nolimits_{\bx {\sim} \hat{\mmP}_{\theta_i}} \left[ f \left(\bx\right) \right] \right] \\
		\ge& \mathop{{\mmE}}\nolimits_{\bx \sim \hat{\mmP}_s} \left[ f_w (\bx) \right] - \frac{1}{N} \sum_i \left[ \mathop{{\mmE}}\nolimits_{\bx {\sim} \hat{\mmP}_{\theta_i}} \left[ f_w \left(\bx\right) \right] \right] \\
		=& \mathop{{\mmE}}\nolimits_{\bx \sim {\mmP}_s} \left[ f_w (\bx) \right] - \frac{1}{N} \sum_i \left[ \mathop{{\mmE}}\nolimits_{\bx {\sim} {\mmP}_{\theta_i}} \left[ f_w \left(\bx\right) \right] \right] 
		- \left( \mathop{{\mmE}}\nolimits_{\bx \sim {\mmP}_s} \left[ f_w (\bx) \right] - \mathop{{\mmE}}\nolimits_{\bx \sim \hat{\mmP}_s} \left[ f_w (\bx) \right] \right)\\
		&- \frac{1}{N} \sum_i \left[ \mathop{{\mmE}}\nolimits_{\bx {\sim} \hat{\mmP}_{\theta_i}} \left[ f_w \left(\bx\right) \right] - \mathop{{\mmE}}\nolimits_{\bx {\sim} {\mmP}_{\theta_i}} \left[ f_w \left(\bx\right) \right] \right]\\
		\ge& W \left({\mmP}_s, {\mmP}_{\theta_1}, \ldots, {\mmP}_{\theta_N}\right) - \epsilon.
	\end{align*}
	Similarly, $ W \left({\mmP}_s, {\mmP}_{\theta_1}, \ldots, {\mmP}_{\theta_N}\right) \ge W \left(\hat{\mmP}_s, \hat{\mmP}_{\theta_1}, \ldots, \hat{\mmP}_{\theta_N}\right) - \epsilon $.
	Therefore, when the number of sample in each domain satisfying $ n \ge \frac{C \kappa \Delta^2 \log(L\kappa/\epsilon) }{\epsilon^2} $, then the following satisfies with the probability at least $ 1-\exp(-\kappa) $,
	\begin{align*}
		\left|W\left(\hat{\mmP}_s, \hat{\mmP}_{\theta_1}, {\ldots}, \hat{\mmP}_{\theta_N}\right) - W(\mmP_s, \mmP_{\theta_1}, {\ldots}, \mmP_{\theta_N}) \right| \leq \epsilon.
	\end{align*}
	We conclude the proof.
\end{proof}

\section{Proof of Lemma \ref{ieqn:constraint}} \label{sec:proof_lemma_Lf}
\emph{
	\textbf{Lemma \ref{ieqn:constraint} \emph{(Constraints relaxation)}} 
	If the cost function $ c(\cdot) $ is measured by $ \ell_2 $ norm, then there exists a constant $ L_f {\ge} 1 $ such that discriminator $ f $ satisfies the following constraint:
	\begin{align}
		\sum_i \frac{\left| f(\bx) - f \left(\hat{\bx}^{(i)} \right) \right| }{ \left\| \bx - \hat{\bx}^{(i)} \right\| } \leq L_f.
	\end{align}
}

\begin{proof}
	When the inequality constraints in Problem \ref{problem:MWGAN} is satisfied, and without loss of generality, we assume that $ \frac{1}{N} \sum \left| f(\bx) - f\left(\bx^{(i)}\right) \right|  \leq c( \bx^{(0)}, \ldots, \bx^{(N)} ) $. 
	Let $ c:= c( \bx^{(0)}, \ldots, \bx^{(N)} ) $, we have 
	\begin{align*}
		\frac{1}{Nc} \min\limits_i \left\| \bx - \hat{\bx}^{(i)} \right\| \sum_i \frac{  \left| f(\bx) - f\left(\hat{\bx}^{(i)}\right) \right|}{\left\| \bx - \hat{\bx}^{(i)} \right\|}  
		\leq \frac{1}{Nc} \sum_i \frac{  \left\| \bx - \hat{\bx}^{(i)} \right\| \left| f(\bx) - f\left(\hat{\bx}^{(i)}\right) \right|}{\left\| \bx - \hat{\bx}^{(i)} \right\|} \leq 1.
	\end{align*}
	Let $ L_f = Nc / \min\limits_i \left\| \bx - \hat{\bx}^{(i)} \right\| \ge 1 $, we conclude the proof.
\end{proof}
In Lemma \ref{ieqn:constraint}, the constant $ L_f $ is related to the cost function $ c $. 
In this sense, it captures the dependency among domains.

\section{Discussions on Lipschitz Condition} \label{sec:discuss_lip}
From the following proposition, the assumption that the potential function is Lipschitz continuous is strong to enforce the inequality constraints.
It would cause misleading results for our problem setting.
\begin{prop} \label{prop:potential_lip}
	If the potential function is Lipschitz continuous, and the cost function is defined as 
	\begin{align}
		c \left( \bx, \bx^{(1)} \ldots, \bx^{(N)} \right) = \sum_{i \in [N]} \left\| \bx - \bx^{i} \right\|,
	\end{align}
	then the potential function must satisfy the inequality constraints, \ie 
	\begin{align}
		\frac{1}{N} \sum_{i \in [N]} \left| f(\bx) - f\left(\bx^{(i)}\right) \right|  \leq c \left( \bx, \bx^{(1)} \ldots, \bx^{(N)} \right).
	\end{align}
\end{prop}
\begin{proof}
	If the potential function is 1-Lipschitz continuous, \ie
	\begin{align}
		\left| f(\bx) - f\left(\bx^{i}\right) \right| \leq \left\| \bx - \bx^{i} \right\|, \quad i \in [N].
	\end{align}
	Then, based on the definition of the potential and for all variables, we have 
	\begin{align}
		\frac{1}{N} \sum_{i \in [N]} \left| f(\bx) - f\left(\bx^{(i)}\right) \right|  \leq \sum_{i \in [N]} \left\| \bx - \bx^{i} \right\|.
	\end{align}
\end{proof}

\newpage
\section{Comparisons with GAN Methods} \label{sec:diff_gan}
\subsection{Differences between MWGAN and WGAN}  
In this paper, the proposed MWGAN essentially differs from WGAN even when $\lambda_i^+ = 1/N$: 
\textbf{{1)}}  MWGAN considers  and incorporates multi-domain correlations into the inequality constraints to improve the \textbf{image translation} performance. WGAN focuses on \textbf{image generation tasks} and cannot directly deal with  multi-domain correlations. 
\textbf{{2)}} The objectives of two methods are different in the formulation. 
\textbf{{3)}} In the algorithm, MWGAN uses gradient penalty to deal with inequality constraints; while WGAN relies on the weight clipping.

\subsection{Comparisons with Image-to-image Translation Methods} 
CycleGAN \cite{ref_zhu2017unpaired} is a two-domain translation method, but it can be used in the multi-domain image translation task.
It means that CycleGAN needs to learn multiple two-domain translation tasks.
Moreover, CycleGAN performs well on the unbalanced translation task, because it independently optimizes multiple individual networks for the multi-domain image translation task.
StarGAN \cite{ref_choi2018stargan} and UFDN \cite{ref_liu2018ufdn} are multi-domain image translation methods, however, they may not exploit multi-domain correlations to achieve good performance on the unbalanced translation task.

\begin{table}[h!]
	\centering
	\caption{Comparisons with image-to-image translation methods.}
	\resizebox{1\textwidth}{!}{
		\begin{tabular}{c|c|c|c|c}
			\hline 
			{Method} & {Unpaired data} & {Multiple domains} & {Multi-domain correlations} & {Unbalanced translation task} \\
			\hline 
			CycleGAN & $\yes$ & $\no $ & $\no $ & $\yes$\\
			StarGAN  & $\yes$ & $\yes$ & $\no $ & $\no $\\
			UFDN     & $\yes$ & $\yes$ & $\no $ & $\no $\\
			MWGAN    & $\yes$ & $\yes$ & $\yes$ & $\yes$\\
			\hline 
		\end{tabular}
	}
	\label{table:comparisons_baselines}
\end{table}

\paragraph{Difference between MWGAN and StarGAN.}
The adversarial learning of MWGAN is different from StarGAN.
Specifically, MWGAN cannot be interpreted as distribution matching between source and a mixture of target distributions.
Let $ \bar{\mmP}_{\theta} $ be a mixture distribution over $ (\mmP_{\theta_1}, \ldots, \mmP_{\theta_N}) $.
Note that $ \bar{\mmP}_{\theta} $ is related to the batch size.
When the batch size is too small, then $ \bar{\mmP}_{\theta} $ cannot guarantee to contain all domains.
StarGAN minimizes the following optimization problem:
\begin{align}\label{obj:stargan}
	\max\nolimits_f \mmE_{\bx \sim \hat{\mmP}_s} [f(\bx)] - \mmE_{\hat{\bx} {\sim} \bar{\mmP}_{\theta}} \left[ f \left(\hat{\bx}\right) \right].
\end{align}
In contrast, MWGAN minimizes the following optimization problem,
\begin{align}\label{obj:mwgan}
	\max\nolimits_f \mathop{{\mmE}}\nolimits_{\bx \sim \hat{\mmP}_s} \left[ f (\bx) \right] 
	- \sum\nolimits_i \lambda_i^+ \mathop{{\mmE}}\nolimits_{\hat{\bx} {\sim} \hat{\mmP}_{\theta_i}} \left[ f \left(\hat{\bx}\right) \right].
\end{align}
When $ \lambda_i^+ {=} 1/N $ and $ \bar{\mmP}_{\theta} $ is uniformly drawn from every target generated distribution, the objective (\ref{obj:stargan}) is equivalent to the objective (\ref{obj:mwgan}). 
However, when $ \lambda_i^+ {\neq} 1/N $ and the batch size is small, the objective (\ref{obj:stargan}) is not equivalent to the objective (\ref{obj:mwgan}). 
Besides, the inequality constraints in MWGAN are related to the correlation among all domains, while StarGAN only considers the source domain and certain target domain. 
Therefore, the adversarial learning of MWGAN is different from StarGAN.

\newpage
\section{Effectiveness of One Potential Function} \label{sec:one_potential}
\subsection{Comparisons between one and multiple potential functions}
We focus on multiple marginal matching, where multiple target domains often contain cross-domain correlations. 
Thus, a shared potential function helps to exploit cross-domain correlations to improve performance (see results in Table \ref{table:comparison_shared_diff_potentials} on the Edge$ \rightarrow $CelebA task). 
Second, training a shared function using entire data of all domains is much easier than training $ N{+}1 $ potentials (one per domain).

\begin{table}[h]
	\caption{Comparisons of shared and $ N{+}1 $ potentials in terms of FID.}
	\label{table:comparison_shared_diff_potentials}
	\centering
	\resizebox{0.5\textwidth}{!}{
		\begin{tabular}{c|c|c|c}
			\hline
			Method      & Black hair & Blond hair & Brown hair \\
			\hline
			$N{+}1$ potentials &    245.25  &    289.56  & 303.04  \\
			\hline
			One potential   &    33.81   &    51.87   & 35.24  \\
			\hline
		\end{tabular}
	}
\end{table}

\subsection{Weight Setting and Performance vs \#domains ($ N $)} \label{subsec:lambda_setting_performance}
With $ \lambda_i^+ = 1/N $, each generator provides equal gradient feedbacks in each target domain and helps to exploit cross-domain correlations in adversarial learning. 
We apply MWGAN on the Edge$ \rightarrow $CelebA translation task with different $ N $. 
From Table \ref{table:performance_vs_domains}, more domains help to improve the performance in terms of FID by exploiting cross-domain correlations.

\begin{table}[h]
	\caption{Performance vs \#domains in FID.}
	\label{table:performance_vs_domains}
	\centering
	\resizebox{0.4\textwidth}{!}{
		\begin{tabular}{c|c|c|c|c}
			\hline
			$ N $ & 2 & 3 & 4 & 5 \\
			\hline
			FID & 58.61 & 38.31 & 33.81 & 32.43 \\
			\hline
		\end{tabular}
	}
\end{table}

\newpage
\section{Toy Dataset}\label{sec:toy_dataset}

\subsection{7 Gaussian Distributions}
In the first row of Figure \ref{fig:comparison_contour}, we generate 7 Gaussian distributions as the real data distribution, where the center of initial distribution (green) is $ (0, 0) $, and the centers of 6 target distributions (red) are $ (3/2, 0) $, $ (-3/2, 0) $, $ (3/4, 3\sqrt{3}/4) $, $ (-3/4, 3\sqrt{3}/4) $, $ (3/4, -3\sqrt{3}/4) $ and $ (-3/4, -3\sqrt{3}/4) $.
For each Gaussian distribution, the variance is $ 0.04 $, and we generate 256 samples.
The synthetic data distribution (orange) is generated from the Gaussian centered at $ (0, 0) $.
\vspace{-5pt}
\subsection{1 Gaussian and 6 Uniform Distributions}
In the second row of Figure \ref{fig:comparison_contour}, we generate 1 Gaussian distribution and 6 uniform distributions as the real data distributions, where the center of initial distribution (green) is also $ (0, 0) $, and the centers of 6 uniform distributions (red) are $ (3/2, 0) $, $ (-3/2, 0) $, $ (3/4, 3\sqrt{3}/4) $, $ (-3/4, 3\sqrt{3}/4) $, $ (3/4, -3\sqrt{3}/4) $ and $ (-3/4, -3\sqrt{3}/4) $.
For each uniform distribution, we generate 256 samples in a square around the center (length is 0.4).
The synthetic data distribution (orange) is generated from the Gaussian centered at $ (0, 0) $.
\vspace{-5pt}
\subsection{Toy Experiment Settings}
We use fully connected neural network architecture for all methods. 
The generator contains 3 hidden layers with 512 units followed by ReLU. 
The discriminator contains 2 hidden layers with 512 units followed by ReLU.
We use Adam as the optimizer with $\beta_1$ = 0.5 and $\beta_2$ = 0.999 and the learning rate of all methods is set to 0.0001. 
The hyper-parameters follow the default setting of these methods.

\section{Details of Classification on CelebA} \label{sec:details_classification_CelebA}
\mo{In the facial attribute translation experiment (\yifan{In Section 6.5 of the main submission})}, we train a classifier on CelebA to obtain a near-perfect accuracy, and test on blond hair, eyeglasses, mustache and pale skin to obtain classifier accuracy of 99.62\%, 99.94\%, 99.76\% and 97.96\%, respectively.
In the same way, we use this classifier to test on synthesized single and multiple attributes for the considered methods.


\newpage

\section{More evaluations with Amazon Mechanical Turk (AMT)} \label{sec:amt}
For more quantitative evaluations, we conduct a perceptual evaluation using AMT to assess the performance on the Edge$\rightarrow$CelebA translation task, following the settings of StarGAN \cite{ref_choi2018stargan}.  
From Table \ref{table:amt_acc}, MWGAN wins significant majority votes for the  best perceptual realism, quality and transferred attributes for all facial attributes.

\begin{table}[h]
	\centering 
	\caption{\footnotesize{AMT perceptual evaluation for each attribute.} } 
	\label{table:amt_acc}
	\resizebox{0.5\textwidth}{!}{
		\begin{tabular}{c|c|c|c}
			\hline
			{Method}  & {Black hair} & {Blond hair} & {Brown hair}  \\
			\hline
			CycleGAN  & 9.7\% & 5.7\% & 9.0\%   \\	
			UFDN      & 13.2\% & 15.8\% & 12.9\%   \\
			StarGAN   & 16.0\% & 21.9\% & 19.4\%   \\
			\hline
			MWGAN     & \textbf{61.1\%} & \textbf{56.6\%} & \textbf{58.7\%}  \\
			\hline
		\end{tabular}
	}
\end{table}

\section{Influences of Inner- and Inter-domain Constraints} \label{sec:influence}

In this section, we evaluate the influences of inner-domain constraints and inter-domain constraints on the edge$ \rightarrow $celebA task, respectively.
Specifically, we compare the FID values with different $ \alpha $ (inner-domain constraint weights) and different $\tau$ (inter-domain constraint weights). 
\mo{The value of $\alpha$ and $\tau$ is selected among [0, 0.1, 1, 10, 100].}
Each experiment only evaluates one constraint and fix other parameters.
To evaluate the influence of $\alpha$, we empirically set $\tau=10$.
Otherwise, we set $\alpha=10$ to evaluate the influence of $\tau$.
The results are shown in Tables \ref{table:alation_edge2hair} and \ref{table:alation_edge2hair2}.

Specifically, when $\alpha{=}0$ or $\tau {=} 0$, MWGAN obtains the worst performance. In other words, when we abandon any one of the inner or inter-domain constraints, we cannot achieve a satisfactory result. This demonstrates the effectiveness of both constraints. 
Besides, MWGAN achieves the best performance when setting both weights to 10.
This means that when setting some reasonable constraint weights, we can achieve a better trade-off between the optimization objective and constraints, and thus obtain better performance.

\begin{table}[h]
	\begin{minipage}{0.48\linewidth}
		\caption{Influence of $ \alpha $ for the inner-domain constraint in terms of FID.}
		\label{table:alation_edge2hair}
		\centering
		\resizebox{0.80\textwidth}{!}{
			\begin{tabular}{c|c|c|c}
				\hline
				{$ \alpha $} & {Black hair} & {Blond hair} & {Brown hair} \\
				\hline
				0   & 316.41 & 334.08 & 325.31 \\	
				0.1 & 263.85 & 317.89 & 300.79 \\
				1   & 109.65 & 109.48 & 136.97 \\
				10  & \textbf{33.81} & \textbf{51.87} & \textbf{35.24} \\
				100 & 55.96 & 71.60 & 66.17 \\	
				\hline
			\end{tabular}
		}
	\end{minipage}
	~~~
	\begin{minipage}{0.48\linewidth}
		
		\caption{Influence of $ \tau $ for the inter-domain constraint in terms of FID.}
		\label{table:alation_edge2hair2}
		\centering
		\resizebox{0.80\textwidth}{!}{
			\begin{tabular}{c|c|c|c}
				\hline
				{$ \tau $} & {Black hair} & {Blond hair} & {Brown hair} \\
				\hline
				0   & 392.87 & 360.17 & 346.16 \\	
				0.1 & 276.07 & 328.16 & 337.11 \\
				1   & 90.75 & 87.99 & 93.18 \\
				10  & \textbf{33.81} & \textbf{51.87} & \textbf{35.24} \\
				100 & 54.30 & 56.44 & 48.03 \\	
				\hline
			\end{tabular}
		}
	\end{minipage}
\end{table}

\section{Influences of the parameter $ L_f $} \label{sec:influence_Lf}
In this section, we evaluate the influences of the parameter $ L_f $ on the edge$ \rightarrow $celebA task.
Specifically, we compare the FID values with different $ L_f $ in the inter-domain constraints. 
The value of $L_f$ is selected among [1, 3, 10, 50], where 3 is the number of domains.
Each experiment only evaluates one constraint and fix other parameters.
The results are shown in Table \ref{table:alation_Lf_edge2hair}.

\cao{
	In Tables \ref{table:alation_Lf_edge2hair}, MWGAN achieves the best performance when setting both weights to 3.
	This means that when setting some reasonable constants, we can achieve a better gradient penalty between the source and each target domain, and thus obtain better performance by exploiting the cross-domain correlations.
}

\begin{table}[h]
	\caption{Influence of the parameter $ L_f $ for the domain constraint in terms of FID.}
	\label{table:alation_Lf_edge2hair}
	\centering
	\resizebox{0.40\textwidth}{!}{
		\begin{tabular}{c|c|c|c}
			\hline
			{$ L_f $} & {Black hair} & {Blond hair} & {Brown hair} \\
			\hline
			1 & 64.17 & 55.98 & 49.19 \\	
			3 & \textbf{33.81} & \textbf{51.87} & \textbf{35.24} \\
			10& 44.12 & 52.46 & 43.64 \\
			50& 79.89 & 91.07 & 79.50 \\
			\hline
		\end{tabular}
	}
\end{table}

\newpage
\section{Network Architecture and More Implementation Details} \label{sec:net_architecture}

\paragraph{Network architecture.}
The classifier $\phi$ shares the same structure except for the output layer with $ f $.
The network architectures of the discriminator and generators of MWGAN are shown in Tables \ref{table:generator_net} and \ref{table:discriminator_net}. 
We split each generator to an encoder and a decoder, where all generators share the same encoder but with different decoders.
For the encoder and decoder network, instead of using batch normalization \cite{ref_guo2018aaai, ref_Ioffe2015icml}, we use instance normalization in all layers except the last output layer of the decoder. 
For the discriminator network, we use PatchGAN network which is made up of fully convolutional networks, and we use Leaky ReLU with a negative slope of 0.01. 
We use the following abbreviations: $h$: the width size of input image, $w$: the height size of input image, $ n_d $: the number of transferred domains(exclude source domain), N: the number of output channels,  K: kernel size, S: stride size, P: padding size, IN: instance normalization.

\paragraph{More implementation details.}
For Loss (\ref{loss:classification}), we use mean square loss and cross-entropy loss for the balanced and imbalanced translation task, respectively.
In the experiments, we find that introducing an identity mapping loss~\cite{ref_zhu2017unpaired} helps improve the quality of generated images on the facial attribute translation task. 
Specifically, the identity mapping loss is defined as: $\mathcal{L}_{\text{idt}}(g_i)=  \mathbb{E}_{\bx \sim \mathbb{\hat{P}}_{t_i}} \left[ \left\| g_i(\bx)-\bx \right\|_1 \right]$, 
where $\mathbb{\hat{P}}_{t_i}$ is an empirical distribution in the $i$-th target domain.
We use the identity mapping loss for all target domains.


\begin{table}[h]
	\small
	\caption{Generator network architecture.}
	\label{table:generator_net}
	\centering
	\begin{tabular}{c|c|c}
		\hline
		\hline
		\multicolumn{3}{c}{\textbf{Encoder}} \\
		\hline
		\hline 
		{Part} & {Input $ {\rightarrow} $ Output shape} & {Layer information} \\
		\hline 
		\multirow{3}[0]{*}{Down-sampling}  
		& $ (h, w, 3) {\rightarrow} (h, w, 64) $ & CONV-(N64, K7x7, S1, P3), IN, ReLU \\	
		& $ (h, w, 64) {\rightarrow} (\frac{h}{2}, \frac{w}{2}, 128) $ & CONV-(N128, K4x4, S2, P1), IN, ReLU \\
		& $ (\frac{h}{2}, \frac{w}{2}, 128) {\rightarrow} (\frac{h}{4}, \frac{w}{4}, 256) $ & CONV-(N256, K4x4, S2, P1), IN, ReLU  \\
		\hline 
		\multirow{3}[0]{*}{Bottleneck}  
		& $ (\frac{h}{4}, \frac{w}{4}, 256) {\rightarrow} (\frac{h}{4}, \frac{w}{4}, 256) $ & Residual Block: CONV-(N256, K3x3, S1, P1), IN, ReLU \\
		& $ (\frac{h}{4}, \frac{w}{4}, 256) {\rightarrow} (\frac{h}{4}, \frac{w}{4}, 256) $ & Residual Block: CONV-(N256, K3x3, S1, P1), IN, ReLU \\
		& $ (\frac{h}{4}, \frac{w}{4}, 256) {\rightarrow} (\frac{h}{4}, \frac{w}{4}, 256) $ & Residual Block: CONV-(N256, K3x3, S1, P1), IN, ReLU \\
		\hline
		\hline
		\multicolumn{3}{c}{\textbf{Decoder}} \\
		\hline
		\hline
		\multirow{3}[0]{*}{Bottleneck}  
		& $ (\frac{h}{4}, \frac{w}{4}, 256) {\rightarrow} (\frac{h}{4}, \frac{w}{4}, 256) $ & Residual Block: CONV-(N256, K3x3, S1, P1), IN, ReLU \\
		& $ (\frac{h}{4}, \frac{w}{4}, 256) {\rightarrow} (\frac{h}{4}, \frac{w}{4}, 256) $ & Residual Block: CONV-(N256, K3x3, S1, P1), IN, ReLU \\
		& $ (\frac{h}{4}, \frac{w}{4}, 256) {\rightarrow} (\frac{h}{4}, \frac{w}{4}, 256) $ & Residual Block: CONV-(N256, K3x3, S1, P1), IN, ReLU \\
		\hline
		\multirow{3}[0]{*}{Up-sampling}
		& $ (\frac{h}{4}, \frac{w}{4}, 256) {\rightarrow} (\frac{h}{2}, \frac{w}{2}, 128) $ & DECONV-(N128, K4x4, S2, P1), IN, ReLU \\
		& $ (\frac{h}{2}, \frac{w}{2}, 128) {\rightarrow} (h, w, 64) $ & DECONV-(N64, K4x4, S2, P1), IN, ReLU \\
		& $ (h, w, 64) {\rightarrow} (h, w, 3) $ & CONV-(N3, K7x7, S1, P3), Tanh \\
		\hline
		\hline
	\end{tabular}
\end{table}

\begin{table}[h]
	\small
	\caption{Discriminator network architecture.}
	\label{table:discriminator_net}
	\centering
	\begin{tabular}{c|c|c}
		\hline 
		\hline
		{Layer} & {Input $ {\rightarrow} $ Output shape} & {Layer information} \\
		\hline 
		Input Layer  & $ (h, w, 3) {\rightarrow} (\frac{h}{2}, \frac{w}{2}, 64) $ & CONV-(N64, K4x4, S2, P1), Leaky ReLU \\
		\hline
		Hidden Layer & $ (\frac{h}{2}, \frac{w}{2}, 64) {\rightarrow} (\frac{h}{4}, \frac{w}{4}, 128) $        & CONV-(N128, K4x4, S2, P1),  Leaky ReLU \\
		Hidden Layer & $ (\frac{h}{4}, \frac{w}{4}, 128) {\rightarrow} (\frac{h}{8}, \frac{w}{8}, 256) $       & CONV-(N256, K4x4, S2, P1),  Leaky ReLU \\
		Hidden Layer & $ (\frac{h}{8}, \frac{w}{8}, 256) {\rightarrow} (\frac{h}{16}, \frac{w}{16}, 512) $     & CONV-(N512, K4x4, S2, P1),  Leaky ReLU \\
		Hidden Layer & $ (\frac{h}{16}, \frac{w}{16}, 512) {\rightarrow} (\frac{h}{32}, \frac{w}{32}, 1024) $  & CONV-(N1024, K4x4, S2, P1), Leaky ReLU \\
		Hidden Layer & $ (\frac{h}{32}, \frac{w}{32}, 1024) {\rightarrow} (\frac{h}{64}, \frac{w}{64}, 2048) $ & CONV-(N2048, K4x4, S2, P1), Leaky ReLU  \\
		\hline
		Output layer ($ f $) & $ (\frac{h}{64}, \frac{w}{64}, 2048) {\rightarrow} (\frac{h}{64}, \frac{w}{64}, 1) $ & CONV-(N1, K3x3, S1, P1) \\
		Output layer ($ \phi $) & $ (\frac{h}{64}, \frac{w}{64}, 2048) {\rightarrow} (1, 1, n_d) $ & CONV-(N$ (n_d) $, K$ \frac{h}{64} $x$ \frac{w}{64} $, S1, P0) \\
		\hline
		\hline
	\end{tabular}
\end{table}	

\newpage
\section{Additional Qualitative Results} \label{sec:more_qualitative_results}

\subsection{Results on CelebA}

\begin{figure*}[h!]
	\centering
	{
		\includegraphics[width=1\linewidth]{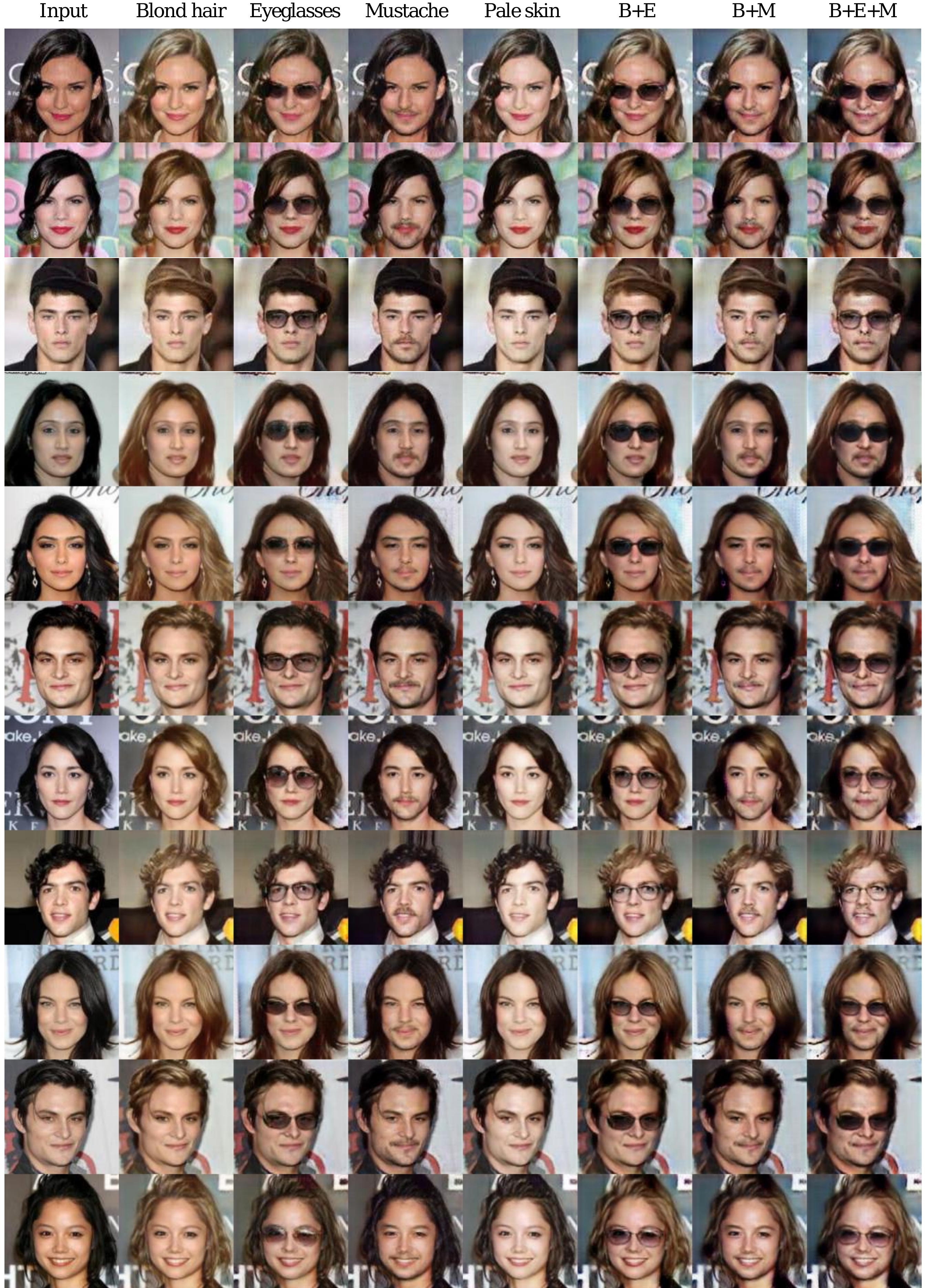}
		\caption{Single and multiple attribute translation results on CelebA.} 
		\label{fig:celeba_add}
	}
\end{figure*}

\newpage
\subsection{Results on Edge$ \rightarrow $CelebA}
\begin{figure*}[h!]
	\centering
	{
		\includegraphics[width=1\linewidth]{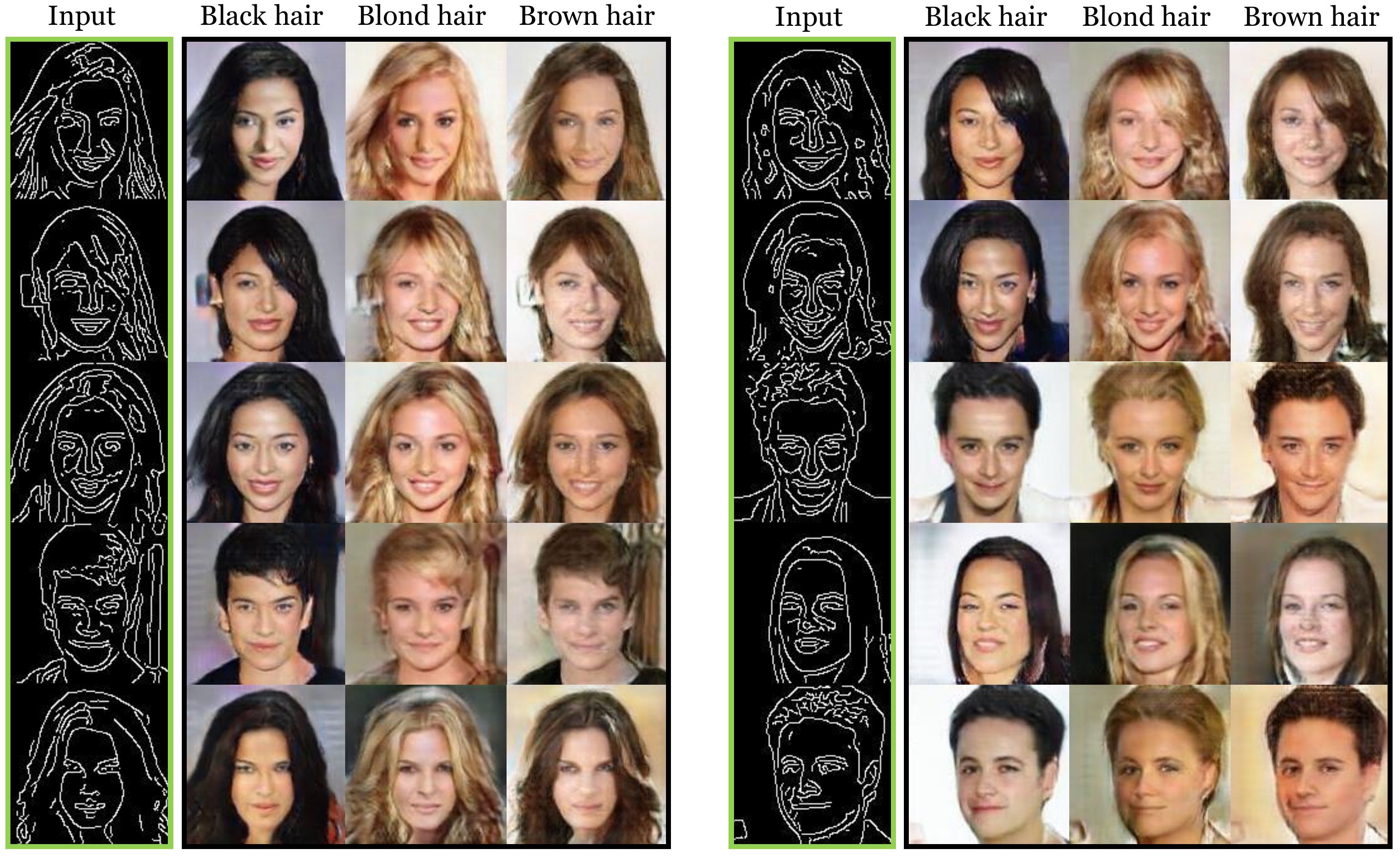}
		\caption{Translation results from edge images to CelebA.}
		\label{fig:edge2celeba_add}
	}
\end{figure*}

\subsection{Results on Painting Translation}
\begin{figure*}[h!]
	\centering
	{
		\includegraphics[width=1\linewidth]{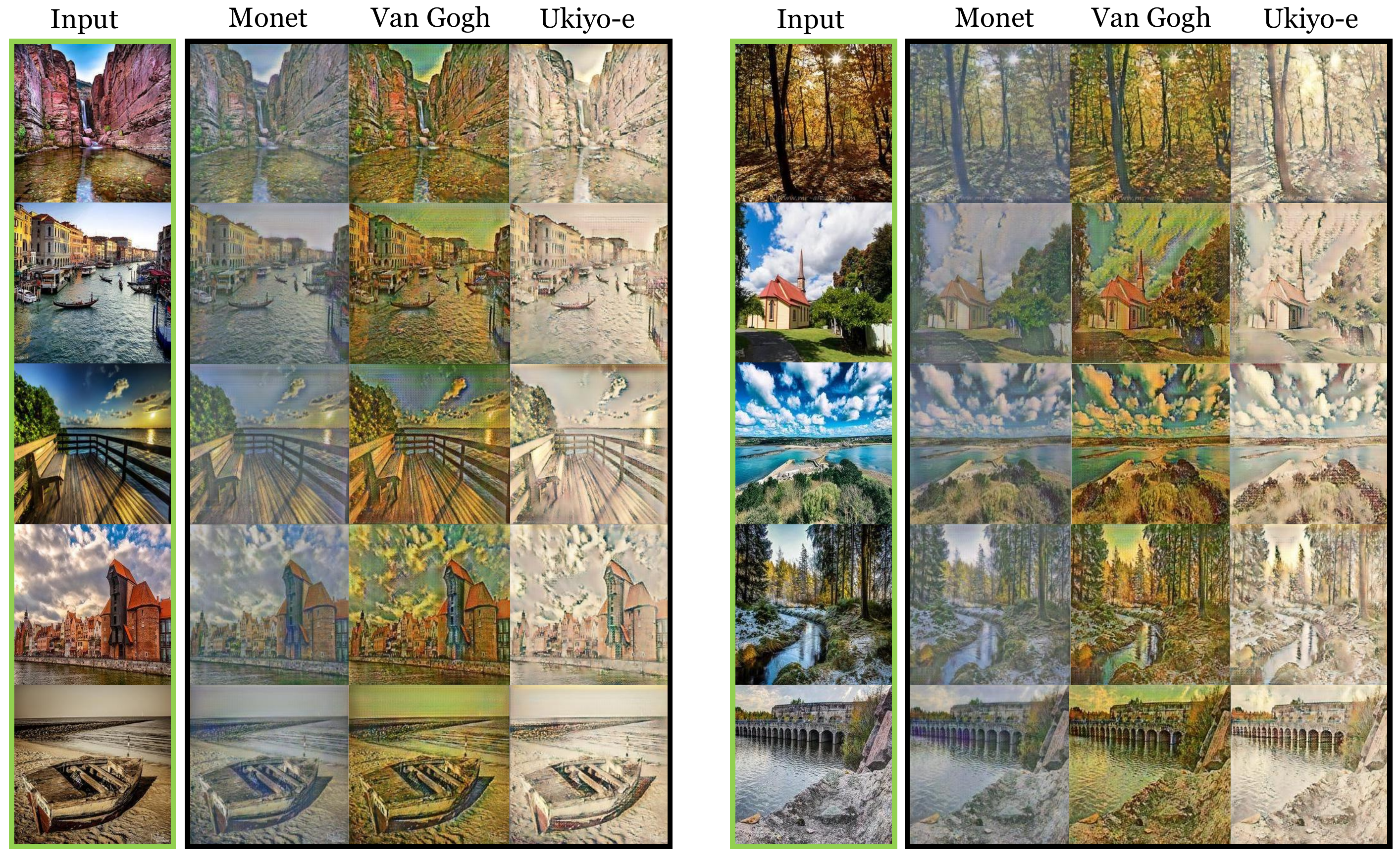}
		\caption{Translation results from real-world images to painting images.}
		\label{fig:photo2art_add}
	}
\end{figure*}

\end{document}